\documentclass{article}

\usepackage{dilab_arxiv}

\usepackage{mathtools}
\usepackage{bbm}
\usepackage{ifthen}
\usepackage{textcomp}
\usepackage{algorithm}
\usepackage{algorithmic}
\usepackage{wrapfig}
\makeatletter
\newcommand{\needspace}[1]{\par\ifdim\pagegoal=\maxdimen\else\dimen@\pagegoal\advance\dimen@-\pagetotal\ifdim\dimen@<#1\newpage\fi\fi}
\makeatother
\usepackage{enumitem}
\usepackage{tikz}
\usetikzlibrary{arrows.meta,positioning}

\usepackage{CJKutf8}

\usepackage[capitalise,nameinlink]{cleveref}

\usepackage{xurl}

\usepackage{amsmath,amsfonts,bm,bbm}

\Crefname{ALC@unique}{Line}{Lines}

\Crefname{algorithm}{Algorithm}{Algorithms}
\Crefname{assumption}{Assumption}{Assumptions}
\Crefname{lemma}{Lemma}{Lemmas}
\Crefname{proposition}{Proposition}{Propositions}
\Crefname{corollary}{Corollary}{Corollaries}
\Crefname{theorem}{Theorem}{Theorems}
\Crefname{definition}{Definition}{Definitions}
\Crefname{remark}{Remark}{Remarks}

\Crefformat{equation}{Eq. #2(#1)#3}
\Crefrangeformat{equation}{Eqs. #3(#1)#4 to #5(#2)#6}
\Crefmultiformat{equation}{Eqs. #2(#1)#3}{ and #2(#1)#3}{, #2(#1)#3}{ and #2(#1)#3}
\Crefrangemultiformat{equation}{Eqs. #3(#1)#4 to #5(#2)#6}{ and #3(#1)#4 to #5(#2)#6}{, #3(#1)#4 to #5(#2)#6}{ and #3(#1)#4 to #5(#2)#6}

\def\1{\mathbbm{1}}

\def\vh{{\bm{h}}}

\def\vu{{\bm{u}}}

\def\vw{{\bm{w}}}

\DeclareMathAlphabet{\mathsfit}{\encodingdefault}{\sfdefault}{m}{sl}
\SetMathAlphabet{\mathsfit}{bold}{\encodingdefault}{\sfdefault}{bx}{n}

\def\gF{{\mathcal{F}}}

\def\gL{{\mathcal{L}}}

\newcommand{\wh}{\widehat}

\newcommand{\R}{\mathbb{R}}

\newcommand{\cP}{\mathcal{P}}

\newif\ifsup\supfalse
\suptrue

\DeclareMathOperator*{\argmax}{argmax}

\newcommand{\BPS}{\texttt{BPS}}

\theoremstyle{plain}
\newtheorem{theorem}{Theorem}
\newtheorem{lemma}{Lemma}

\theoremstyle{definition}

\theoremstyle{remark}

\title{Optimal Skill Selection for LLM Agents \\ with Provable Bicriteria Guarantees}
\runningtitle{Optimal Skill Selection for LLM Agents}
\date{arXiv preprint, \today}

\paperlogo{\includegraphics[height=1.5cm]{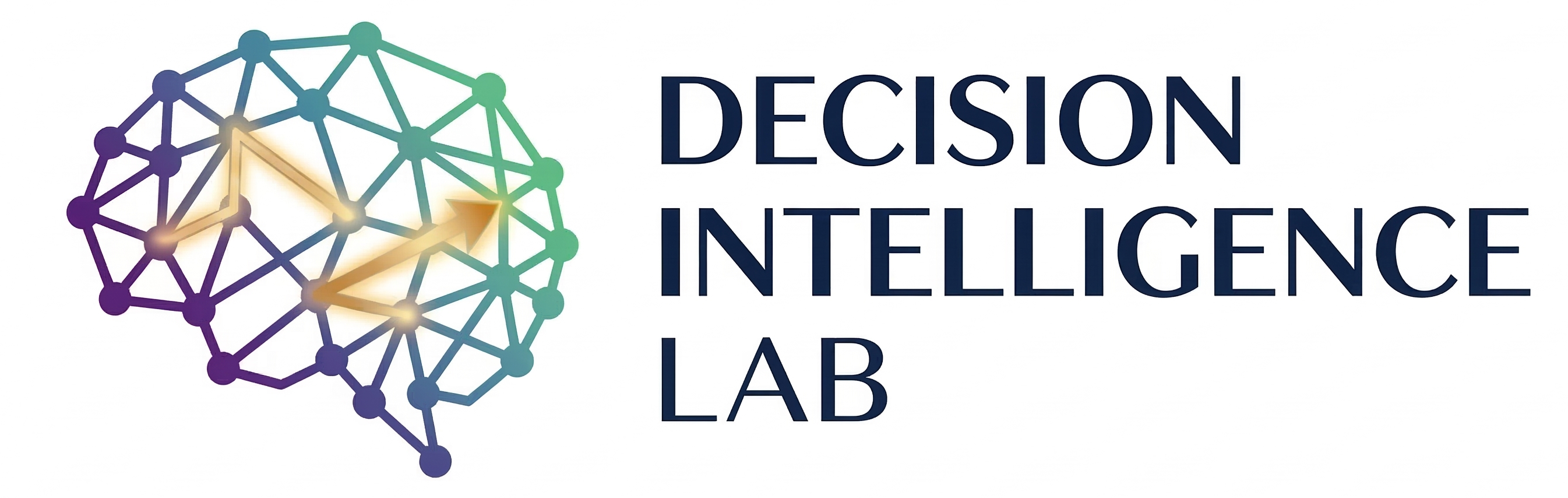}}

\author{
  Yu Chen$^{1,\star}$, Ruishuo Chen$^{1,\star}$, Xun Wang$^{1}$, Zhuoran Li$^{1}$, and Longbo Huang$^{1\,\text{\faEnvelope}}$
  \\[0.3em]\normalfont
  $^1$Institute for Interdisciplinary Information Sciences, Tsinghua University
  \\
  \text{\faEnvelope}\ Correspondence: longbohuang@tsinghua.edu.cn \\
  $^\star$ Equal contribution, listed in random order.
}

\begin{document}
\begin{CJK*}{UTF8}{gbsn}

\maketitle
\thispagestyle{fancy}

\begin{abstract}
Loading reusable skill documents into a bounded context window has become a primary way large language model (LLM) agents acquire task-specific capabilities, which makes skill selection a first-order determinant of task performance and token cost. Yet current agents score skills independently by semantic relevance and assemble the set by top-$k$ or greedy packing, with no quality guarantee or cost awareness on the selected set. Redundant or poorly chosen skills then waste scarce context tokens and can even degrade performance. 
In this paper, we present a theory-grounded and practical framework for budgeted skill selection.
We give the first model of how skill sets shape execution outcomes, capturing complementary capability coverage and diminishing returns from redundancy through a monotone submodular benefit, while accounting for context degradation with a linear token penalty under a hard budget.
Based on this model, we develop Best Prefix Selection (\BPS{}), a polynomial-time algorithm, and prove, to our knowledge, the first performance guarantee for skill selection: a bicriteria $(1-1/e,1)$ approximation whose benefit coefficient is optimal in polynomial time. 
We construct a controlled testbed based on BigCodeBench to isolate the effect of skill selection on execution success. On it, \BPS{} with a learned capability encoder reaches a success rate of $0.65$, and the strongest baselines need at least $28\%$ more tokens to reach $0.60$.
\end{abstract}

\section{Introduction}
\label{sec:introduction}

Large language model (LLM) agents increasingly rely on reusable skill documents to acquire task-specific capabilities beyond their parametric knowledge \cite{yang2026surveyagentskills}, and public skill registries already list tens of thousands of installable skills \cite{cho2026skillret,gao2026skillreducer}.
Modern production agents such as Codex \cite{openai2026skills} and Claude Code \cite{anthropicagentskills} deploy skills through a two-stage mechanism of \emph{selection} and \emph{execution}: the LLM reviews each installed skill's metadata (name and description) and selects skills according to the query, then loads only the selected skill documents into its context window to solve the task. 

However, both stages are limited by the model's finite context window: the context cost of selection scales with the size of the installed skill library, whereas execution consumes additional context for the selected skill documents and task input. Deployments are therefore confined to small libraries, and Codex, for instance, allots skill metadata at most $2\%$ of the context window \cite{openai2026skills}. Moreover, redundant content in the context disperses the model's attention and degrades execution \cite{liu2024lost,modarressi2025nolima}, and even curated skills push success below the no-skill baseline on $13$ of $87$ benchmark tasks \cite{li2026skillsbench}.

These limitations have motivated dedicated skill-selection mechanisms, from per-skill retrieval and routing to set-aware packing and context construction \cite{gan2025ragmcp,fore2024geckopt,zheng2026skillrouter,zheng2026skillselectserve,li2026skillsinjector}. Without a principled formulation, however, they remain heuristic. Skills are ranked by semantic relevance, a learned preference, or graph proximity, and the set is cut by top-$k$ \cite{qu2024colt,li2026agentskillos}, truncation \cite{liu2026graphofskills}, or greedy packing \cite{zheng2026skillselectserve}. The resulting set carries no guarantee, and both its context cost and its capability overlap and complementarity are handled ad hoc.

\begin{figure}
    \centering
    \includegraphics[width=.8\linewidth]{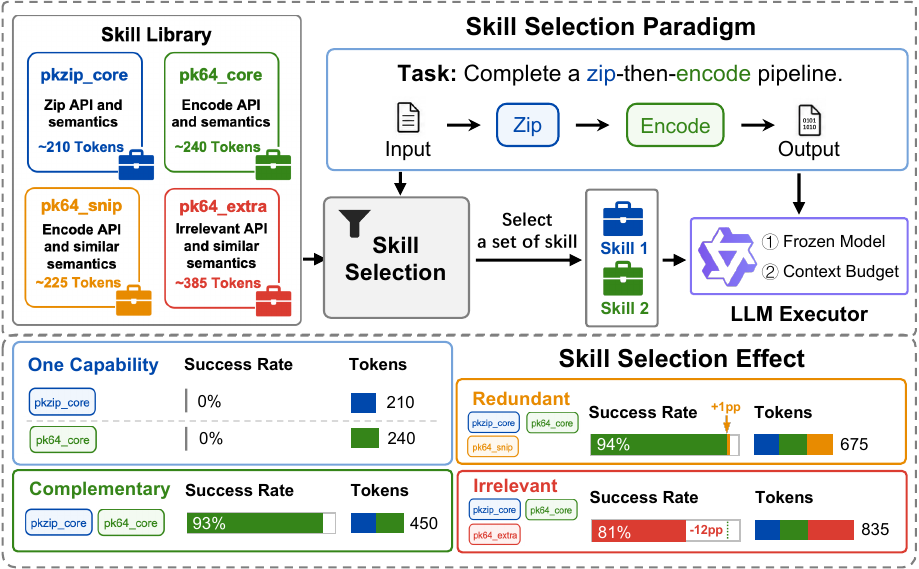}
    \caption{\textbf{Top:} skill selection in coding agents: a selector picks skills from a library and injects them into a frozen LLM executor. \textbf{Bottom:} selection is set-level: sets covering the required capabilities succeed, while redundant and irrelevant skills spend context budget for little or negative return.}
    \label{fig:moti}
\end{figure}
Yet these capability relationships are critical to effective skill selection.
\cref{fig:moti} illustrates a common skill-selection paradigm for coding agents \cite{openai2026skills,anthropicagentskills}: given a task, the system selects skills from a library and provides them to a frozen LLM executor. We evaluate Qwen3-32B \cite{yang2025qwen3} in this paradigm on tasks whose private APIs are accessible only through the selected skills, so execution depends directly on selection quality. We find that skills covering only one capability achieve zero success, whereas complementary skills covering both reach a $93\%$ success rate. After both capabilities are covered, adding a redundant \texttt{pk64\_snip} skill consumes $225$ additional tokens but improves the success rate by only $1$ percentage point, while adding a semantically related but task-irrelevant \texttt{pk64\_extra} skill reduces it by $12$ percentage points. These results establish skill selection as a budgeted set-level decision rather than a per-skill ranking. 

Motivated by these observations, we formalize skill selection as a budgeted set decision within the selection--execution workflow of LLM agents (\cref{sec:model}). We use a structured value model that can be fitted directly from execution outcomes while keeping the executor fixed.
Each skill supplies latent capabilities and the query demands some of them; a concave response within each capability dimension captures diminishing returns from redundant coverage, while separate dimensions reward complementary coverage, so the gross benefit $G(S)$ is monotone submodular. Injected tokens are charged a linear context penalty $\kappa\ell(S)$ under a budget $B$:
\begin{equation}
  \max_{S:\ell(S)\le B}
  F(S) := G(S)-\kappa\ell(S),
  \label{eq:intro-selection}
\end{equation}
an instance of regularized submodular maximization under a knapsack constraint.

Solving this optimization with a provable guarantee, however, faces fundamental barriers.
Without the penalty, the problem contains monotone submodular knapsack maximization, which admits no approximation better than $1-1/e$ unless $\mathrm{P}=\mathrm{NP}$ \cite{feige1998threshold}. 
With the penalty, the objective may be negative, ruling out constant multiplicative approximation and motivating a bicriteria guarantee \cite{nikolakaki2021efficient,harshaw2019distorted}.
Existing results for regularized submodular maximization under a hard budget obtain weaker approximation coefficients \cite{gong2024budget,guo2026budgetedprofit,zhang2026streaming} or incur an additive precision-dependent loss \cite{perrault2021rsds}.

We develop Best Prefix Selection (\BPS{}, \Cref{alg:bps}), a practical polynomial-time algorithm for budgeted skill selection, which can be deployed as the selection stage in most popular LLM agent systems. We prove that \BPS{} achieves a $(1-1/e,1)$ bicriteria guarantee for the fitted objective in \Cref{eq:def-hat-F}, with the best possible benefit coefficient and the full context penalty.
To our knowledge, this work provides the first structured model of how the selected skill set shapes execution outcomes and the first skill-selection algorithm with a provable performance guarantee. On a controlled testbed with real LLM executions, the \BPS{} pipeline outperforms all evaluated skill-selection baselines in average task success, while using fewer tokens than the strongest competitors.
We summarize our main contributions as follows.
\begin{itemize}[leftmargin=1.4em, itemsep=1pt, topsep=2pt]
\item \textbf{New Formulation.} We formalize skill selection as a regularized submodular maximization problem under a token budget constraint in \cref{sec:model}.
The structured objective makes redundancy, complementarity, and context cost explicit, and its parameters can be fitted from execution outcomes. We validate the model on real executions in \cref{sec:evaluation:e1}. 
\item \textbf{Optimal Bicriteria Guarantee.} We develop a polynomial-time algorithm for the skill selection problem (\BPS{}, \Cref{alg:bps}), and prove a $(1-1/e,1)$ bicriteria guarantee for \BPS{} with the best possible benefit coefficient (\cref{thm:main}), using budget-aligned interpolation (\cref{sec:guarantee:proof}). In experiments, \BPS{} reaches the best value found on $98\%$ of selection instances with up to $124$ skills, compared with $56\%$ for the best baseline (\Cref{app:opt}).
\item \textbf{Experiments on Real Executions.} We construct a controlled testbed based on BigCodeBench \citep{zhuo2024bigcodebench} to isolate the effect of skill selection on execution success. On it, our structured objective \Cref{eq:def-hat-F} predicts execution success on unseen skill sets and recovers their hidden capability coverage. On tasks held out from training its neural capability encoder, the sets \BPS{} selects reach \(0.65\) measured task success, above the \(0.60\) that the executor's own selection and a retriever fine-tuned on the same feedback reach on \(28\%\) and \(36\%\) more tokens, and above every released router and retriever at \(0.22\) or less (\cref{sec:evaluation:e1,sec:evaluation:e2}).
\end{itemize}

\section{Related Work}
\label{sec:related}

\textbf{Skill selection for LLM agents.}~\label{sec:related:skill}
Skill retrieval and routing \cite{cho2026skillret,zheng2026skillrouter,su2026sra,xiao2026skillsight,wang2026r3} narrow a large skill library to a candidate pool for downstream selection, following a pipeline inherited from tool use \cite{qin2023toolllm,qu2024colt}.
SkillsInjector \cite{li2026skillsinjector} learns how many skills to inject and renders them jointly; SkillSelect-Serve \cite{zheng2026skillselectserve} greedily packs itemwise scores under token and deployment constraints; Graph-of-Skills \cite{liu2026graphofskills} expands dependency-aware bundles under a context cap yet leaves its budgeted objective unsolved; GoSkills \cite{zeng2026goskills} and SkillComposer \cite{zhao2026skillcomposer} assemble bounded skill groups and autoregressive subsets. To our knowledge, none of these heuristics states a provable guarantee for the injected set.
Related methods provide guarantees for different selection problems: PACMS \cite{ghulyani2026pacms} applies facility-location coverage to accumulated session content under a token knapsack, and the knapsack composer of \cite{yuan2025knapsack} admits agentic components online with a competitive ratio.
In contrast, we formulate skill selection for a fixed executor as regularized submodular maximization under a hard token budget, explicitly modeling set-level redundancy and complementarity.

\textbf{Regularized submodular maximization.}~\label{sec:related:submodular}
Constrained submodular maximization is a classical template for placing content and services on capacity-limited resources \cite{golrezaei2012femtocaching,poularakis2019joint}. For the pure benefit, density greedy with seed enumeration attains the tight $1-1/e$ \cite{sviridenko2004note,kulik2021refined,feldman2023practical}. Subtracting a linear penalty changes the problem's character: the objective can be negative, which rules out any constant multiplicative approximation \cite{nikolakaki2021efficient}, and distorted greedy \cite{harshaw2019distorted} recovers a $(1-1/e,1)$ bicriteria guarantee only under a cardinality constraint.
Under a knapsack constraint, existing guarantees for the regularized objective include $(1/4,1)$ \cite{gong2024budget}, $((1-1/e)/2,1/2)$ \cite{guo2026budgetedprofit}, and $(1/8-\epsilon,1)$ in one streaming pass \cite{zhang2026streaming}. Other results incur an additive $\kappa\epsilon$ loss with $O(B/\epsilon)$ budget levels \cite{perrault2021rsds} or return a fractional solution \cite{feldman2021guessfree}. ParetoGreedy \cite{vombatkere2026pareto} records all greedy prefixes as \BPS{} does, but its instance-dependent Pareto analysis does not establish a uniform $(1-1/e,1)$ guarantee for $G-\kappa\ell$.
Our result establishes a uniform $(1-1/e,1)$ guarantee for an integral solution without additive loss.

\section{Model and Optimization Problem}
\label{sec:model}

We formulate skill selection as maximizing an improvement in task score under a context budget.
Let $\gL = \{s_1, \cdots, s_L\}$ denote the skill library, where $s_i$ is the $i$-th skill document.
Given a query $q$, the selection stage chooses a subset of skills, indexed by $S \subseteq [L]$, for a downstream executor $E$, held fixed throughout.
Injecting the documents of the selected skills into the context window of $E$, let $J_E(q,S)\in\mathbb R$ be the score of the execution result on task $q$, with higher scores indicating better performance. The score can be an expected task grade, an expected human rating, or a utility of success probability. We define the true score improvement over execution without skills as
\begin{equation}
  F_E^\star(q,S) := J_E(q,S)-J_E(q,\varnothing). \label{eq:pop-effect}
\end{equation}
The selection stage would ideally maximize $F_E^\star(q,S)$, but $F_E^\star$ is a black box. We therefore approximate it with a structured objective (\cref{sec:model:objective}), pose selection as maximizing it under a hard token budget (\cref{sec:model:problem}), and fit it from execution records, with empirical evidence supporting its use for skill selection (\cref{sec:model:validity}).

\subsection{Key Observations and Structured Objective}
\label{sec:model:objective}

The structured model rests on three observations. \textbf{(i)} Context value is query-dependent and set-level: a skill's contribution is not an intrinsic per-skill score but depends on whether its capabilities match those the query requires \cite{didolkar2024metacognitive,an2023skillbased,xu2024lars} and complement those the selected set already supplies \cite{gupta2023coverage}. \textbf{(ii)} More context is not uniformly beneficial: irrelevant content distracts execution \cite{shi2023large}, redundant items add little coverage \cite{gupta2023coverage,kumari2024end}, and longer inputs can hurt even when the relevant evidence is present \cite{liu2024lost}. \textbf{(iii)} For a fixed executor and agent framework, the selected documents must fit the residual context budget \cite{ye2023compositional,kumari2024end,qu2024colt}.

\textbf{Structured Capability Benefit.}~
Following observations (i) and (ii), we model the positive contribution of skills through a latent capability space with $d$ dimensions. For each skill $s_i \in \gL$, we assume there exists a latent capability supply vector $\vu_i = (u_{i,1}, \cdots, u_{i,d}) \in \R_+^d$, where $u_{i,k}$ quantifies how much capability skill $s_i$ supplies in the $k$-th dimension, e.g., operating git or analyzing logs.

To model complementarity and redundancy, a nondecreasing concave function $h_k$ with $h_k(0)=0$ turns the total supply into the coverage $h_k\left(\sum_{i\in S}u_{i,k}\right)$ of $S$ in dimension $k$, so that supplies in the same dimension overlap redundantly while those in different dimensions remain complementary.

A latent capability demand vector $\vw^q=(w_1^q,\ldots,w_d^q)\in\mathbb R_+^d$ weights the query's demand for each capability dimension. We model the gross score benefit as a monotone submodular function
\begin{equation}\label{eq:def-G}
  G_E(q, S) := \sum_{k=1}^{d} \eta^E_k w_k^q \cdot  h_k\left(\lambda^E_k \sum_{i\in S} u_{i,k}\right),
\end{equation}
where $\eta^E_k \ge 0$ and $\lambda^E_k \ge 0$ are executor-specific parameters that calibrate $E$'s sensitivity to the $k$-th capability dimension.
This is the weighted-coverage form used in document summarization and data selection \cite{lin2011class,kirchhoff2014submodularity}, with capability dimensions as features.

\textbf{Degradation Penalty and Hard Context Budget.}~
Following observation (ii), injected documents impose a cost that grows with context length. Let $\ell_i>0$ be the token length of skill document $s_i$ once injected, and $\ell(S) := \sum_{i\in S}\ell_i$ the total injected length. We approximate the resulting score degradation by $c_E(S) := \kappa_E\ell(S)$, where $\kappa_E\ge0$ is the frozen executor's per-token sensitivity in the chosen score units. Length-sensitive execution motivates this first-order approximation: compressing skill-document bodies improves execution quality \cite{gao2026skillreducer}, and focused skills outperform larger, exhaustive ones \cite{li2026skillsbench}.

Moreover, observation (iii) imposes a hard feasibility limit on skill selection. We take the budget $B$ to be the residual token budget available to skill documents in the agent framework. The hard budget defines the feasible family $\gF_B = \{S\subseteq[L]: \ell(S)\le B\}$.

\subsection{Optimization Problem in Skill Selection}
\label{sec:model:problem}

Given the structured model, we formalize the selection stage as a constrained optimization problem. The selection objective combines the capability benefit with the degradation penalty,
\begin{equation}\label{eq:def-F}
  F_E(q, S) := G_E(q, S) - c_E(S),
\end{equation}
and the selection problem maximizes it over the feasible family:
\begin{equation}\label{eq:selection}
  \max_{S \in \gF_B}  F_E(q, S).
\end{equation}
Problem \Cref{eq:selection} is computationally hard: setting $\kappa_E = 0$ and specializing the responses $h_k$ to truncated sums recovers budgeted maximum coverage \cite{khuller1999budgeted}, an NP-hard slice of monotone submodular knapsack maximization \cite{sviridenko2004note}. The penalty can also make $F_E$ non-monotone and negative, so guarantees for monotone submodular maximization do not apply directly. \cref{sec:guarantee} develops an algorithm with a bicriteria guarantee that treats the benefit and the penalty asymmetrically.

\subsection{Model Fitting and Empirical Validity}
\label{sec:model:validity}
\begin{figure}[t]
\centering
{
\includegraphics[width=0.65\linewidth]{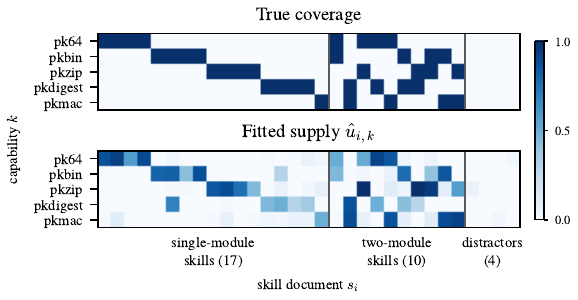}
\caption{Parameter recovery. Columns are the $31$ skill documents, grouped by the private modules they document: one, two, or none (distractors, which document look-alike APIs the tasks never use). Rows are the five capabilities, one per private module. Top: the true coverage from human labels. Bottom: the supply $\wh u_{i,k}$ fitted from pass/fail outcomes alone, scaled by its largest entry, with latent dimensions matched to capabilities by the best permutation. Both panels share the $0$--$1$ scale.}
\label{fig:recovery}}
\end{figure}

The structured objective \Cref{eq:def-F} is built on latent quantities. Among its primitives, only the token length $\ell_i$ for each skill document $s_i$ is observable.
The capability supplies $\vu_i$, the query demand $\vw^q$, and the executor calibration parameters $\eta^E_k$, $\lambda^E_k$, $\kappa_E$ admit no direct measurement.
Two questions therefore decide whether the model is usable in practice: whether these latent quantities can be estimated from data, and whether an accurately estimated objective leads to a well-chosen skill set.

For the first question, the key observation is that the selection objective never requires the latent factors individually. In \Cref{eq:def-G}, the demand $w^q_k$ enters only through the product $\eta^E_k w^q_k$, and the supply $u_{i,k}$ only through $\lambda^E_k u_{i,k}$.
Defining the \emph{effective demand} and \emph{effective supply}
\begin{equation}\label{eq:def-effective}
  \widetilde w^q_k := \eta^E_k w^q_k, \quad \widetilde u_{i,k} := \lambda^E_k u_{i,k}, \quad \forall k \in [d],
\end{equation}
the gross benefit \Cref{eq:def-G} rewrites exactly as $G_E(q,S) = \sum_{k=1}^{d} \widetilde w^q_k\, h_k\bigl(\sum_{i\in S} \widetilde u_{i,k}\bigr)$.
We therefore estimate the effective quantities directly with two \emph{capability encoders}: a demand encoder $\wh{\bm{\psi}}_E: q \mapsto \widehat\vw^q$ and a supply encoder $\wh{\bm{\phi}}_E: s_i \mapsto \widehat\vu_i$, where $\widehat\vw^q$ and $\widehat\vu_i$ are estimates of the effective demand $\widetilde\vw^q$ and effective supply $\widetilde\vu_i$, respectively. Each encoder is instantiated as a text encoder with a nonnegative output layer, so that the fitted objective inherits the monotone submodularity established in \cref{sec:model:objective}.
The context sensitivity is calibrated as $\wh\kappa_E$ jointly with the encoders on execution records. Then we define the fitted objective
\begin{equation}\label{eq:def-hat-F}
  \widehat F_E(q, S) := \left\langle \wh{\bm{\psi}}_E(q), \vh\left(\sum_{i\in S} \wh{\bm{\phi}}_E(s_i)\right) \right\rangle - \wh{\kappa}_E \cdot \ell(S),
\end{equation}
where $\vh(\cdot)$ applies $\{h_k\}$ componentwise. At decision time, the selection layer solves the \emph{fitted selection problem}, which is the deployable counterpart of \Cref{eq:selection} and is defined by
\begin{equation}\label{eq:selection-fitted}
  \max_{S \in \gF_B} \widehat F_E(q, S).
\end{equation}
Solving \Cref{eq:selection-fitted} provides a practical substitute for \Cref{eq:selection}, with the latent quantities estimated from execution records. Its usefulness depends on whether \Cref{eq:def-hat-F} captures how skill combinations affect the frozen executor's performance. We assess this on the testbed of \cref{sec:evaluation:testbed}, examining its predictions, the capability structure it learns, and the execution outcomes of the sets it selects.

We detail the empirical evaluation in \cref{sec:evaluation:e1}. Specifically, we fit \Cref{eq:def-hat-F} to pass/fail outcomes and test its predicted success probabilities. With minimal parameters, our model achieves the highest accuracy among the compared models. On the top panel of \Cref{fig:recovery}, we label the $31$ skills in the testbed by their coverage of $5$ different capabilities. The bottom panel shows the supply $\wh u_{i,k}$ fitted on all $5{,}204$ measured task--set pairs of the testbed (\cref{app:testbed}), with latent dimensions matched to capabilities by the best permutation. The fitted supplies recover the manually annotated skill-capability coverage with an AUC of $0.996$. In \cref{sec:evaluation:e2} we also validate that optimizing the modeled objective improves actual execution.

\section{Skill Selection with Provable Guarantees}
\label{sec:guarantee}

We now present the selection algorithm and its per-instance guarantee for the fitted selection problem \Cref{eq:selection-fitted}. For a fixed query $q$, frozen executor $E$, and budget $B$, the fitted objective in \Cref{eq:def-hat-F} is $\wh F_E(q,S) = \wh G_E(q,S) - \wh\kappa_E\ell(S)$, where $\wh G_E(q,S)$ is the fitted capability benefit and $\wh\kappa_E$ is the calibrated context sensitivity.
Our analysis uses only that $\wh G_E(q,\cdot)$ is normalized ($\wh G_E(q,\varnothing)=0$), nondecreasing, and submodular, that all lengths $\ell_i$ are positive, and that $\wh\kappa_E \ge 0$.

The objective $\wh F_E(q,S) = \wh G_E(q,S) - \wh\kappa_E\ell(S)$ raises three difficulties at once. First, although $\wh G_E(q,\cdot)$ is monotone, $\wh F_E(q,\cdot)$ can be non-monotone and negative, so a selection rule must be allowed to stop early or output the empty set. In particular, the classical greedy analysis for monotone submodular maximization no longer applies to $\wh F_E(q,\cdot)$. Second, skills have heterogeneous lengths under a single knapsack constraint, so locally dense choices are not enough. Third, the executor accepts only integral skill sets, so fractional reasoning must eventually land on an integral candidate.
These difficulties call for a new analysis of approximate greedy algorithms.

\begin{algorithm}[t]
  \caption{Best Prefix Selection (\BPS{})}
  \label{alg:bps}
  \begin{algorithmic}[1]
  \REQUIRE task $q$, executor $E$, skill library $\gL$, budget $B$, fitted benefit oracle $\wh G_E(q,\cdot)$, fitted context sensitivity $\wh\kappa_E$.
  \ENSURE selected skill set $S_{\BPS}$.
  \STATE Discard every skill $i$ with $\ell_i > B$; initialize the \emph{prefix pool} $\cP \leftarrow \emptyset$.
  \FOR{each seed $A \subseteq [L]$ with $|A| \le 2$ and $\ell(A) \le B$}
    \STATE $S \leftarrow A$
    \STATE Add prefix $S$ to $\cP$. \COMMENT{each seed opens a prefix chain}
    \WHILE{some $i \notin S$ fits, i.e., $\ell(S) + \ell_i \le B$}
      \STATE $i^\star \leftarrow \argmax_{i \notin S,\, \ell(S)+\ell_i\! \le\! B} \left(\wh G_E(q,\{i\}\!\cup\! S)\! -\! \wh G_E(q,S)\right)\!/\ell_i$
      \STATE $S \leftarrow S \cup \{i^\star\}$.
      \STATE Add prefix $S$ to $\cP$. \COMMENT{record every prefix in chain}
    \ENDWHILE
  \ENDFOR
  \STATE $S_{\BPS} \leftarrow \argmax_{S \in \cP}\, \wh F_E(q,S)$. \COMMENT{choose the best prefix}
  \RETURN $S_{\BPS}$
  \end{algorithmic}
\end{algorithm}

We present our algorithm, \emph{Best Prefix Selection} (\BPS{}, \cref{alg:bps}), a partial-enumeration density-greedy procedure with seed size two.
In Lines 2--4, the algorithm enumerates all feasible seeds of size at most two. Then it grows a density-greedy chain from each seed by iteratively adding the skill with the highest marginal benefit per token (Lines 5--9), and records every feasible prefix encountered along each chain. Finally, it returns the single \emph{best} recorded prefix, the one maximizing the fitted objective $\wh F_E(q,\cdot)$ (Line 11).

\cref{alg:bps} implements the standard partial-enumeration density greedy for monotone submodular knapsack \cite{khuller1999budgeted,sviridenko2004note,kulik2021refined}. 
Under a monotone benefit, the endpoint of each chain dominates all its prefixes, so the known $1-1/e$ analyses compare only chain endpoints. 
Under the non-monotone fitted objective $\wh F_E(q,\cdot)$, however, the endpoint need not be the best candidate, since the optimum can sit strictly inside a chain.
Inspired by \cite{vombatkere2026pareto}, we record every prefix and select the best one over the full collection by $\wh F_E(q,\cdot)$; this best-prefix selection step gives \BPS{} its name. 

The theoretical guarantee for \BPS{} is given below, with full proof in \cref{app:proof-main}. 

\begin{theorem}[Bicriteria $(1-1/e,\,1)$-approximation guarantee]
  \label{thm:main}
  Let $\wh G_E(q,\cdot)$ be normalized ($\wh G_E(q,\varnothing) = 0$), nondecreasing, and submodular, let $\ell_i > 0$ for all $i \in [L]$, and let $\wh\kappa_E \ge 0$. The \BPS{} output $S_{\BPS}$ of \cref{alg:bps} satisfies $S_{\BPS} \in \gF_B$ and
  \begin{equation}
    \wh F_E(q,S_{\BPS}) \ge \left(1 - \frac{1}{e}\right)\wh G_E(q,T) - \wh\kappa_E\ell(T) \quad\forall T \in \gF_B.
    \label{eq:guarantee}
  \end{equation}
\end{theorem}

\Cref{eq:guarantee} gives \BPS{} a bicriteria guarantee for $\wh F_E$, our learned approximation to the black-box objective $F_E^\star$, with coefficients $1-1/e$ on benefit and $1$ on penalty.
It controls optimization of the fitted objective but does not bound model error or, by itself, performance on $F_E^\star$.
We evaluate the model's selection quality in \cref{sec:evaluation:e1} and demonstrate improved execution success with fewer tokens for the resulting \BPS{} pipeline on our controlled testbed in \cref{sec:evaluation:e2}.

\textbf{Tightness.}~The benefit coefficient $\alpha = 1-1/e$ in \Cref{eq:guarantee} cannot be improved under standard complexity assumptions (any $(1-1/e+\epsilon)$-approximation is NP-hard \cite{feige1998threshold}). \Cref{thm:main} therefore provides a polynomial-time bicriteria guarantee with the best possible benefit coefficient for maximizing a monotone submodular benefit minus a linear penalty under a knapsack constraint.

\textbf{Comparison to prior work.}~The guarantees closest to our setting address the same regularized objective under the same knapsack constraint, and each concedes what \cref{thm:main} does not: weaker approximation coefficients \cite{gong2024budget,guo2026budgetedprofit,zhang2026streaming}, an additive precision-dependent loss \cite{perrault2021rsds}, or a fractional output \cite{feldman2021guessfree}.
\texttt{ParetoGreedy}~\cite{vombatkere2026pareto} generates the same candidates as \BPS{}, but its analysis gives instance-dependent Pareto bounds rather than a uniform $(1-1/e,1)$ guarantee for the regularized objective.
In contrast, \cref{thm:main} attains the $(1-1/e,1)$ guarantee directly for an integral solution, with no additive loss.

\textbf{Time Complexity.}~\Cref{alg:bps} runs in $O(dL^4)$ time, where $d$ is the capability dimension which can be selected a priori and $L$ the library size: $O(L^2)$ seeds each grow a chain of at most $L$ density steps, and each step scans $O(L)$ candidates at $O(d)$ cost per marginal evaluation. In practice $L$ is the size of the
shortlist left by a high-recall retrieval stage \cite{cho2026skillret,zheng2026skillrouter,gan2025ragmcp}, not of the whole registry.

\textbf{Empirical optimization quality.}~
We also evaluate \BPS{} on $900$ skill-selection instances with libraries of $31$ to $124$ skills tiled from an objective fitted to real executor outcomes in \cref{sec:evaluation:e1}. \BPS{} reaches the best value found on $98\%$ of them and outperforms all baselines, including \citet{gong2024budget}, whose share falls to $15\%$ at $124$ skills, and \citet{guo2026budgetedprofit} at $39\%$ (\cref{app:opt}).

\section{Evaluation}
\label{sec:evaluation}

In this section, we evaluate the effectiveness of our proposed structured value model and the \BPS{} algorithm. Details of the experimental setup and implementation are provided in \cref{app:details}.
We first introduce a controlled testbed in \Cref{sec:evaluation:testbed}. Using this testbed, we evaluate the validity of the structured objective (\Cref{eq:def-hat-F}) in \Cref{sec:evaluation:e1}. We further propose an end-to-end skill selection protocol based on our methods in \Cref{sec:evaluation:e2}.

\subsection{Controlled Testbed for Skill Selection}
\label{sec:evaluation:testbed}

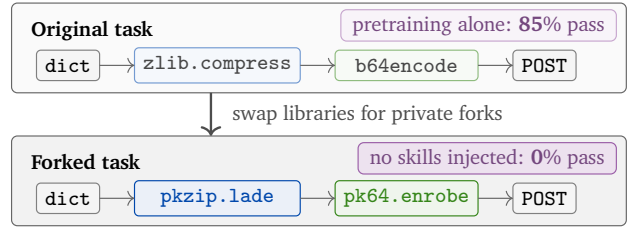
\begin{wrapfigure}[10]{r}{0.5\textwidth}
\centering
\vspace{-12pt}
\ifdefined\stepcolzip\else\newlength{\stepcolzip}\newlength{\stepcolenc}\fi
\settowidth{\stepcolzip}{\scriptsize\ttfamily zlib.compress}
\settowidth{\stepcolenc}{\scriptsize\ttfamily pk64.enrobe}
\addtolength{\stepcolzip}{6pt}
\addtolength{\stepcolenc}{6pt}
\resizebox{\linewidth}{!}{%
\begin{tikzpicture}[font=\scriptsize, every node/.style={outer sep=0pt},
  pipe/.style={draw=black!50, rounded corners=1.5pt, inner sep=3pt,
               align=center, fill=gray!8, font=\scriptsize\ttfamily},
  rate/.style={rounded corners=2pt, inner xsep=4pt, inner ysep=2pt,
               font=\scriptsize}]
\definecolor{famzip}{HTML}{0048AC}
\definecolor{fam64}{HTML}{358519}
\definecolor{ratehue}{HTML}{9B5FA8}
\draw[black!60, rounded corners=2pt, fill=black!2] (-11.9,0.50) rectangle (-4.5,-0.58);
\node[anchor=north west, font=\scriptsize\bfseries] at (-11.78,0.40)
  {Original task};
\node[rate, anchor=north east, draw=ratehue!55, fill=ratehue!7,
      text=ratehue!75!black] at (-4.62,0.42)
  {pretraining alone: $\mathbf{85\%}$ pass};
\node[pipe, anchor=west] (a1) at (-11.60,-0.25) {dict};
\node[pipe, right=0.42 of a1, minimum width=\stepcolzip, draw=famzip!55,
      fill=famzip!4, text=black!72] (a2) {zlib.compress};
\node[pipe, right=0.42 of a2, minimum width=\stepcolenc, draw=fam64!55,
      fill=fam64!4, text=black!72] (a3) {b64encode};
\node[pipe, right=0.42 of a3] (a4) {POST};
\draw[->, black!60] (a1) -- (a2);
\draw[->, black!60] (a2) -- (a3);
\draw[->, black!60] (a3) -- (a4);
\draw[->, thick, black!70] (-9.50,-0.58) -- (-9.50,-1.10)
  node[midway, right=4pt, font=\scriptsize] {swap libraries for private forks};
\draw[black!60, rounded corners=2pt, fill=black!5] (-11.9,-1.10) rectangle (-4.5,-2.18);
\node[anchor=north west, font=\scriptsize\bfseries] at (-11.78,-1.20)
  {Forked task};
\node[rate, anchor=north east, draw=ratehue, fill=ratehue!18,
      text=ratehue!75!black] at (-4.62,-1.18)
  {no skills injected: $\mathbf{0\%}$ pass};
\node[pipe, anchor=west] (b1) at (-11.60,-1.85) {dict};
\node[pipe, right=0.42 of b1, minimum width=\stepcolzip, draw=famzip,
      fill=famzip!8, text=famzip] (b2) {pkzip.lade};
\node[pipe, right=0.42 of b2, minimum width=\stepcolenc, draw=fam64,
      fill=fam64!8, text=fam64] (b3) {pk64.enrobe};
\node[pipe, right=0.42 of b3] (b4) {POST};
\draw[->, black!60] (b1) -- (b2);
\draw[->, black!60] (b2) -- (b3);
\draw[->, black!60] (b3) -- (b4);
\end{tikzpicture}
}
\caption{Forking a BigCodeBench task: standard libraries are replaced by private forks.
}
\label{fig:testbed}
\end{wrapfigure}

On existing benchmarks, success cannot be attributed entirely to the selected skills, because the executor may already solve the task from pretraining alone. Qwen3-32B~\citep{yang2025qwen3} solves the original BigCodeBench~\citep{zhuo2024bigcodebench} versions of our tasks in $85\%$ of rollouts without any skill. 
To isolate the effect of skill selection, we build a testbed based on BigCodeBench by replacing standard library APIs with private variants (\cref{app:testbed}).

As shown in \cref{fig:testbed}, we fork BigCodeBench and replace its standard libraries with five private modules, one per capability family. The constructed skill library consists of skills documenting a single module or a pair of modules, as well as similar-looking but unrelated distractors. We exclude tasks that can be solved without using the constructed skills. On this testbed, the frozen executor Qwen3-32B records no successes in $100$ executions without injected skills, which allows us to cleanly study the effect of skill selection on downstream task success and token cost. For each task and selected skill set used below, we run Qwen3-32B as the frozen executor and compute success rates from these real execution outcomes, more than $90{,}000$ runs in total.

\subsection{Validity of the Structured Objective}
\label{sec:evaluation:e1}
\begin{figure}[!t]
\centering
\includegraphics[width=\linewidth]{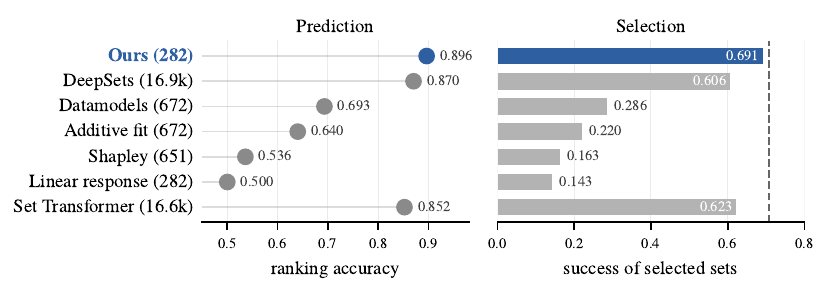}
\caption{Value-model comparison, parameter counts in parentheses. \emph{Left:} pairwise ranking accuracy on the held-out sets of the test tasks. \emph{Right:} mean success rate of the selected sets on executions held out from fitting, with the empirical ceiling dashed.}
\label{fig:models}
\end{figure}

We evaluate whether higher scores from our structured value model in \Cref{eq:def-hat-F} correspond to higher execution success rates for the same task. The model is denoted by \textbf{Ours} in \Cref{fig:models}. For this experiment, we learn a demand vector for each task and a supply vector for each skill in a five-dimensional capability space. The model has $282$ learnable parameters. We detail the parameterization in \cref{app:parameterization} and its connection to execution success and fitting procedure in \cref{app:fit}.

We compare our structured value model with several interpretable baselines. Additive fit and Shapley scores~\cite{ghorbani2019datashapley} give each skill an independent per-task value and sum it over the set. 
Datamodels~\cite{ilyas2022datamodels} fits a regularized linear model that predicts execution success rates from binary indicators of which skills are selected.
DeepSets~\cite{zaheer2017deepsets} and Set Transformer~\cite{lee2019settransformer} are black-box neural set regressors.
We also include Linear Response, an ablation that sets $h(x)=x$ in \Cref{eq:def-hat-F}. 
We fit every model using execution outcomes for all measured skill sets from the $11$ training tasks and $20$ randomly sampled skill sets from each of the $10$ test tasks. We then evaluate each model's predictions on the remaining $2{,}806$ task--set pairs from the test tasks. \Cref{app:validity} gives the implementation details and the training/test split.

\textbf{Prediction.} 
For a given task $q$, good models report higher scores for skill sets that are more likely to succeed on the frozen executor. 
We take every pair of skill sets from the same test task whose measured success rates differ by more than $0.25$ and count how often the tested model ranks them correctly (\cref{fig:models}, left). Our structured model in \Cref{eq:def-hat-F} achieves the highest ranking accuracy with the fewest parameters, exceeding DeepSets and Set Transformer by $0.026$ and $0.044$, respectively. The two neural baselines each use approximately $60$ times as many parameters. Supplementary extrapolation and unseen-doc experiments test generalization to larger sets and new combinations of individually observed skills (\cref{app:validity,fig:protocols}).

\textbf{Selection.} 
We further evaluate each model for skill selection.
For each selection instance, every model selects the highest-scoring candidate skill set that satisfies the token budget.
\Cref{fig:models} (right) reports the selected sets' mean success rate across $204$ instances from eight test tasks, using executions held out from fitting.
The dashed line marks the empirical ceiling, averaging the highest measured execution success rate among feasible sets in each instance.
\Cref{app:validity} details the candidate pools, rollout split, and averaging over three random seeds.
Our structured value model achieves $98\%$ of this ceiling and outperforms all baselines, supporting its effectiveness for skill selection on this testbed.

\subsection{End-to-End Selection Effectiveness}
\label{sec:evaluation:e2}
\begin{wrapfigure}[21]{r}{0.5\linewidth}
\vspace{-0.6\baselineskip}
\centering
\includegraphics[width=\linewidth]{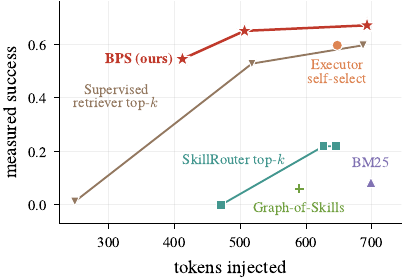}
\caption{End-to-end skill selection: injected tokens versus measured execution success, averaged over the eight test tasks and three token budgets. From left to right, the stars show \BPS{} with the neural capability encoder at $\rho=2,1,0$, and each top-$k$ curve shows $k=1,2,3$.}
\label{fig:selection}
\end{wrapfigure}
The end-to-end test pairs the eight test tasks with budgets of $500$, $750$ and $1{,}000$ tokens, giving $24$ instances. It executes the skill set selected by each method on the frozen executor and measures success on executions that no fit or training used. The lookup tables of \cref{sec:evaluation:e1} cannot run on an unseen task, so we train a neural capability encoder that projects frozen \texttt{text-embedding-v4}~\cite{zhang2025qwen3embedding} embeddings of the task prompt and of each skill into $64$ latent capability dimensions. We warm it up on annotated covering skill sets and then train it online on pass/fail feedback from the executor, on tasks other than the one being tested, and \BPS{} selects on this encoder (\cref{app:e2e}).

We compare against the skill selectors in use today: BM25, the lexical retriever of the skill-retrieval literature~\cite{su2026sra}; the dense retriever our encoder embeds with, \texttt{text-embedding-v4}, fine-tuned on the same annotations and execution feedback as the encoder and ranking each document by its own score (\cref{app:e2e}); SkillRouter~\cite{zheng2026skillrouter}, a trained retrieve-and-rerank router; Graph-of-Skills~\cite{liu2026graphofskills}, which diffuses over an offline skill graph; and the executor selecting for itself from the skill catalogue. Both released systems run from their authors' code on our library. The two rankers are read at budgeted top-$k$, the first $k$ documents in rank order that fit the budget, for $k=1,2,3$. \BPS{} has the matching control in a per-token price added to the fitted context sensitivity $\wh\kappa_E$ of \cref{eq:def-hat-F}, a deployment setting that fixes where on the token--success trade-off it operates. We set this price to $\rho/B$, so that filling the budget costs an additional $\rho$ at every budget, and report $\rho\in\{0,1,2\}$. \Cref{fig:selection} plots the measured success of every rule against the tokens it injects.

Averaged over the three budgets, \BPS{} attains the highest measured success at every token level (\cref{fig:selection}). Without the added price it reaches $0.67$ on $693$ tokens, and at $\rho=1$ it still reaches $0.65$ on only $506$. The fine-tuned retriever, trained on the same supervision, reaches $0.53$ at $k=2$ on about as many tokens as \BPS{} at $\rho=1$ and needs $36\%$ more to reach $0.60$ at $k=3$. The executor selecting for itself reaches $0.60$ with $28\%$ more, while the released routers and the lexical retriever stay at or below $0.22$. Both rankers fail at $k=1$ because every test task exercises at least two modules, and the top document covers only one. The routers and retrievers rarely select distractors but pick skills whose text matches the task, which need not cover the capabilities it exercises. \BPS{} scores a set by the capabilities its documents jointly supply, and so covers what the task needs with fewer tokens.

\section{Conclusion}
\label{sec:conclusion}

We studied the skill selection problem for LLM agents under a hard token budget.
We proposed a learnable model that captures complementary capability coverage, redundancy, and context cost.
Based on this model, we developed \BPS{}, a polynomial-time selection algorithm with a $(1-1/e,\,1)$ bicriteria guarantee for the structured objective.
On our controlled testbed, the structured model recovers skill-capability coverage and achieves the highest ranking accuracy among the baselines.
With a learned capability encoder, \BPS{} achieves higher execution success on held-out tasks with fewer injected tokens than a retriever trained on the same supervision and the executor's own selection.
We discuss the limitations and future directions in \cref{app:future-work}.

\bibliography{references}
\bibliographystyle{dilab_ref}

\makeappendixtoc
\appendix
\crefalias{section}{appendix}\crefalias{subsection}{appendix}
\crefname{appendix}{Appendix}{Appendices}\Crefname{appendix}{Appendix}{Appendices}

\section{\texorpdfstring{Proof of \cref{thm:main}}{Proof of the Main Theorem}}
\label{app:proof-main}\label{sec:guarantee:proof}

Throughout this section, fix the query $q$, executor $E$, and budget $B$.
For brevity, we write $\wh G(S) = \wh G_E(q,S)$, $\wh F(S) = \wh F_E(q,S)$, and $\wh\kappa = \wh\kappa_E$, and let $\alpha = 1 - 1/e$.

\textbf{Proof sketch.}
We first outline the argument, then give it in full. Fix any feasible $T$. If $|T| \le 3$, the seed enumeration of \cref{alg:bps} alone yields the claim. Otherwise the argument has three steps.
\begin{enumerate}
  \item[(i)] Order $T$ greedily by marginal benefit and take the first two items in this greedy order as the seed $J$, which \cref{alg:bps} enumerates; the residual function $f(U) := \wh G(J \cup U) - \wh G(J)$ is again monotone submodular.
  \item[(ii)] Following \cite{kulik2021refined}, trace the residual benefit of the density chain grown from $J$ as a piecewise-affine trajectory $V(u)$ of the added length $u$, and lower-bound it by a piecewise-exponential function $\varphi$. Evaluating at $r$, the length of $T \setminus J$ minus its longest item, gives $\wh G(J) + V(r) \ge \alpha\,\wh G(T)$; removing the longest item is what keeps every remaining item of $T$ feasible while the chain runs up to $r$.
  \item[(iii)] \emph{Budget-aligned interpolation.} The point $u = r$ generally falls strictly inside a chain segment, so $V(r)$ is a convex combination of the residual benefits of two consecutive recorded prefixes. Because the penalty $\wh\kappa\,\ell(\cdot)$ is linear in the same length coordinate, the same combination applied to $\wh F$ equals $\wh G(J) + V(r) - \wh\kappa(\ell(J) + r) \ge \alpha\,\wh G(T) - \wh\kappa\,\ell(T)$, so one of the two prefixes attains the bound. Since \BPS{} returns the $\wh F$-maximizer over all recorded prefixes, the bound holds for its single output and every $T$.
Step (iii) handles the non-monotonicity of $\wh F$: the chain endpoint need not attain the required bound. Because the penalty is linear in length, the bound at $u = r$ transfers to one of the two adjacent recorded prefixes.
\end{enumerate}

\subsection{Seed construction and residual function}
\label{app:seed}
We first write $\wh G(i \mid S) = \wh G(S \cup \{i\}) - \wh G(S)$ for the marginal benefit. 
We order the items of $T$ greedily by nonincreasing marginal benefit with an arbitrary fixed tie-breaking rule:
\begin{equation}\label{eq:comp-order}
\begin{aligned}
  t_j &\in \argmax_{i \in T \setminus \{t_1,\cdots,t_{j-1}\}} \wh G\left(i \mid \{t_1,\cdots,t_{j-1}\}\right), \\
  m_j &:= \wh G\left(t_j \mid \{t_1,\cdots,t_{j-1}\}\right), \quad \forall j \in [|T|].
\end{aligned}
\end{equation}

If $|T| \le 2$, then \cref{alg:bps} enumerates $T$ itself as a seed, and we have $\wh F(T) = \wh G(T) - \wh\kappa\ell(T) \ge \alpha\wh G(T) - \wh\kappa\ell(T).$ 

We set the feasible seed set $J$ with $\ell(J) \le B$ as 
\begin{equation}
\label{eq:seed-def}
  J := \{t_1,t_2\}.
\end{equation}
Then \cref{alg:bps} enumerates it in Line 2 and records it in Line 4, so $J \in \cP$. 
If $|T| = 3$: $\wh G(J) = m_1 + m_2 \ge \frac{2}{3}(m_1 + m_2 + m_3) = \frac{2}{3}\wh G(T) \ge \alpha\wh G(T)$ and $\ell(J) \le \ell(T)$ give $\wh F(J) \ge \alpha\,\wh G(T) - \wh\kappa\,\ell(T)$.

In the following, we assume $|T| \ge 4$. Our goal is to prove that on the chain grown from $J$, some recorded prefix $S_T$ satisfies $\wh F(S_T) \ge \alpha\wh G(T) - \wh\kappa\ell(T)$.

Inspired by the analysis in \cite{sviridenko2004note}, we define the residual function
\begin{equation}\label{eq:def-residual}
  f(U) := \wh G(J \cup U) - \wh G(J), \quad \forall U \subseteq [L].
\end{equation}
By the properties of $\wh G$, $f$ is normalized ($f(\varnothing)=0$), nonnegative, nondecreasing, and submodular. 
Let $P = T \setminus J$ denote the comparator residual. For every $v \in P$, we have $f(\{v\}) = \wh G(\{v\} \cup J) - \wh G(J) = \wh G(v \mid J) \le \wh G(v \mid \{t_1\}) \le m_2$, where the first inequality uses submodularity and the second uses the greedy ordering \Cref{eq:comp-order}. Therefore, we have
\begin{equation}\label{eq:seed-two-marginals}
  \begin{aligned}
    \wh{G}(J) &= \wh{G}(\{t_1, t_2\}) - \wh{G}(\{t_1\}) + \wh{G}(\{t_1\}) = m_1 + m_2 \ge 2m_2 \ge 2f(\{v\}), \quad \forall v \in P.
  \end{aligned}
\end{equation}

We choose $v^\star \in P$ as the item with maximum token length, i.e., $v^\star \in \argmax_{v\in P} \ell(\{v\})$. Set 
\begin{equation}\label{eq:def-r}
  r = \ell(P \setminus \{v^\star\}).
\end{equation}
Then we can show that the density chain starting with seed $J$ must have total token length larger than $r$.

\subsection{Bounding-function domination}
Run the density chain of \Cref{alg:bps} from seed $J$, and let $A_0 = \varnothing \subsetneq A_1 \subsetneq \cdots \subsetneq A_m$ be its accepted additions, so the actual recorded prefixes are $J \cup A_j$.
Write $\{a_j\} = A_j \setminus A_{j-1}$ for the $j$-th accepted item. Following \cite{kulik2021refined}, we define the piecewise-affine residual-benefit trajectory $V$ by $V(0) = 0$ and, for $\ell(A_{j-1}) \le u \le \ell(A_j)$,
\begin{equation}\label{eq:trajectory}
  V(u) := f(A_{j-1})  + \bigl(u - \ell(A_{j-1})\bigr) \frac{f(a_j \mid A_{j-1})}{\ell(\{a_j\})}.
\end{equation}
The function $V: [0, \ell(A_m)] \to \R_{\ge 0}$ traces the residual benefit accumulated by the density chain as a function of the total added length $u$ beyond the seed $J$. 
At each breakpoint $u = \ell(A_j)$, the trajectory evaluates to $V(\ell(A_j)) = f(A_j)$, the exact residual benefit of the $j$-th recorded prefix. 
Between consecutive breakpoints, $V$ interpolates linearly. Then we have the following lemmas.
\begin{lemma}[Trajectory coverage]\label{lem:coverage}
  $V(r)$ is well-defined; that is, $\ell(A_m) \ge r$.
\end{lemma}
\begin{proof}
  If $P \subseteq A_m$, then $\ell(A_m) \ge \ell(P) > r$ by definition in \Cref{eq:def-r}.
  Otherwise pick $v \in P \setminus A_m$. 
  When the chain stops, adding $v$ is infeasible. Then $\ell(A_m) + \ell(\{v\}) > B - \ell(J)$. Since $\ell(\{v\}) \le \ell(\{v^\star\})$ and $\ell(T) \le B$, we have $\ell(A_m) > B - \ell(J) - \ell(\{v\}) \ge \ell(T) - \ell(J) - \ell(\{v^\star\}) = r$.
\end{proof}
\begin{lemma}[Residual feasibility before $r$]\label{lem:feasible}
  Consider an accepted segment whose left endpoint $u_0 = \ell(A_{j-1})$ satisfies $u_0 < r$. Then every item $v \in P \setminus A_{j-1}$ is feasible at that point.
\end{lemma}
\begin{proof}
  By maximality of $\ell(\{v^\star\})$, we have $\ell(A_{j-1} \cup \{v\}) = u_0 + \ell(\{v\}) < r + \ell(\{v^\star\}) = \ell(P) = \ell(T) - \ell(J) \le B - \ell(J)$. Therefore, $v$ is feasible when the algorithm selects $a_j$.
\end{proof}
\Cref{lem:feasible} is where removing the longest residual item pays off.

We adapt the bounding-function technique of \cite{kulik2021refined}. Partition $P$ into two nonempty blocks $\{v^\star\}$ and $R = P \setminus \{v^\star\}$, and order them as $(X_1, X_2)$ so that
\begin{equation}\label{eq:block-order}
  \frac{f(X_1)}{\ell(X_1)} \ge \frac{f(X_2)}{\ell(X_2)}.
\end{equation}
Set $d_j = \ell(X_j)$ for $j = 1, 2$, and define the block densities 
\begin{equation}
  \rho_1 = \frac{f(X_1)}{d_1}, \quad \rho_2 = \frac{f(X_2 \mid X_1)}{d_2},
\end{equation}
where $f(X_2 \mid X_1) = f(X_1 \cup X_2) - f(X_1)$. Submodularity and \Cref{eq:block-order} give $\rho_1 \ge \rho_2$. Write
\begin{equation}\label{eq:def-d}
  D = \ell(P) = d_1 + d_2.
\end{equation}
If $\rho_2 > 0$, let $D_1 = d_1 \ln \frac{\rho_1}{\rho_2}$. If $\rho_2 = 0$, set $D_1 = +\infty$. Define a continuous function $\varphi: [0, \infty) \to \R_{\ge 0}$ by
\begin{equation}\label{eq:phi-def}
  \varphi(u) =
  \begin{cases}
    f(X_1)\left(1 - \exp({-u/d_1})\right), & 0 \le u < D_1, \\
    f(P) - \rho_2 D \exp\left(-\dfrac{u - D_1}{D}\right), & u \ge D_1.
  \end{cases}
\end{equation}

The first branch of $\varphi$ approaches $f(X_1)$ at exponential rate $1/d_1$, while the second approaches $f(P)$ at rate $1/D$. When $D_1$ is finite, the two expressions agree there.

To use $\varphi$ in the proof, we need to relate it both to the greedy trajectory and to the comparator's total benefit. \Cref{lem:domination} shows that $V(u) \ge \varphi(u)$ throughout $[0,r]$, making $\varphi$ a lower bound on the residual benefit along the trajectory. \Cref{lem:static} then shows that the seed benefit together with $\varphi(r)$ is at least $\alpha\wh G(T)$. Combining the two gives $\wh G(J) + V(r) \ge \alpha\wh G(T)$, the benefit bound that the next subsection converts into a bound on the fitted objective at a recorded prefix. We prove these lemmas in \cref{app:proof-domination,app:proof-static}.

\begin{lemma}[Trajectory bound]\label{lem:domination}
  $V(u) \ge \varphi(u), \forall u \in [0, r]$.
\end{lemma}

\begin{lemma}[Static bound at $r$]\label{lem:static}
$\wh G(J) + \varphi(r) \ge \alpha \wh G(T)$.
\end{lemma}

\subsection{Budget-aligned interpolation}
\label{sec:guarantee:interpolation}
\Cref{lem:domination,lem:static} give
\begin{equation}\label{eq:traj-at-r}
  \wh G(J) + V(r) \ge \wh G(J) + \varphi(r) \ge \alpha\,\wh G(T).
\end{equation}
Choose consecutive recorded prefixes $A_{j-1}, A_j$ whose segment contains $r$. There exists $\lambda \in [0,1]$ with $V(r) = (1-\lambda)\,f(A_{j-1}) + \lambda\,f(A_j),$ and $r = (1-\lambda)\,\ell(A_{j-1}) + \lambda\,\ell(A_j)$.
Because the penalty $\wh\kappa\,\ell(\cdot)$ is the same linear function of the same length coordinate, the two interpolations combine: $(1-\lambda)\wh F(J \cup A_{j-1}) + \lambda\wh F(J \cup A_j) = \wh G(J) + V(r) - \wh\kappa(\ell(J) + r)$. Since $\ell(J) + r = \ell(T) - \ell(\{v^\star\}) \le \ell(T)$, \Cref{eq:traj-at-r} gives
$(1-\lambda)\wh F(J \cup A_{j-1}) + \lambda\wh F(J \cup A_j) \ge \alpha\wh G(T) - \wh\kappa(\ell(T) - \ell(\{v^\star\})) \ge \alpha\wh G(T) - \wh\kappa\ell(T)$.
A convex combination of two numbers is at most their maximum, so at least one of the two recorded prefixes satisfies
\begin{equation}\label{eq:witness}
  \wh F(J \cup A_s) \ge \alpha\wh G(T) - \wh\kappa\ell(T),
  \quad s \in \{j-1, j\}.
\end{equation}

Therefore, in every case, some recorded candidate $S_T \in \cP$ satisfies $\wh F(S_T) \ge \alpha\,\wh G(T) - \wh\kappa\,\ell(T)$. Since $S_{\BPS}$ maximizes $\wh F$ over $\cP$ (Line~11 of \cref{alg:bps}), $\wh F(S_{\BPS}) \ge \wh F(S_T)$. The seed and witness prefix depend on $T$, but every seed of size at most two is enumerated and one selection maximizes $\wh F$ over all recorded candidates, so the single output $S_{\BPS}$ satisfies \Cref{eq:guarantee} for every $T \in \gF_B$. This statement finishes the proof of \cref{thm:main}. \hfill $\blacksquare$

\subsection{Auxiliary lemmas}
\label{app:missing-proof}

Throughout the remaining subsections we use the notation introduced above:
the seed $J = \{t_1, t_2\}$ of \Cref{eq:seed-def}, the residual function $f$
of \Cref{eq:def-residual}, the comparator residual $P = T \setminus J$ with
$|T| \ge 4$, the longest residual item $v^\star$, the remainder
$R = P \setminus \{v^\star\}$, the lengths $p := \ell(\{v^\star\})$ and
$r = \ell(R)$ of \Cref{eq:def-r}, the block order $(X_1, X_2)$ of
\Cref{eq:block-order} with $d_j = \ell(X_j)$ and densities
$\rho_1 = f(X_1)/d_1$, $\rho_2 = f(X_2 \mid X_1)/d_2$, the total residual
length $D = \ell(P) = d_1 + d_2$ of \Cref{eq:def-d}, the crossover point
$D_1$, the bounding function $\varphi$ of \Cref{eq:phi-def}, the trajectory
$V$ of \Cref{eq:trajectory}, and $\alpha = 1 - 1/e$. Since $|T| \ge 4$ we
have $|P| \ge 2$, so both blocks $\{v^\star\}$ and $R$ are nonempty; all item
lengths are positive, hence
\begin{equation}\label{eq:app-positive-lengths}
  p > 0, \qquad r > 0, \qquad d_1 > 0, \qquad d_2 > 0, \qquad D = p + r.
\end{equation}
Recall from \cref{app:seed} that $f$ is normalized
($f(\varnothing) = 0$), nonnegative, nondecreasing, and submodular, and that
submodularity with \Cref{eq:block-order} gives $\rho_1 \ge \rho_2 \ge 0$.

We first record two auxiliary lemmas whose proofs are elementary computations; both are used repeatedly below.

\begin{lemma}[Residual marginals and the trajectory]\label{lem:app-residual}
  \leavevmode
  \begin{enumerate}
    \item[(i)] For every $U \subseteq [L]$ and $v \in [L]$,
      $f(v \mid U) = \wh G(v \mid J \cup U) \ge 0$.
    \item[(ii)] $V$ is continuous, piecewise affine, and nondecreasing on
      $[0, \ell(A_m)]$; it satisfies $V(\ell(A_j)) = f(A_j)$ for every
      $j$, and it is defined on all of $[0, r]$.
    \item[(iii)] Consider the chain step that selects $a_j$ from the state
      $J \cup A_{j-1}$, and let
      $\sigma_j = f(a_j \mid A_{j-1})/\ell_{a_j}$ denote its density. Then
      $f(v \mid A_{j-1})/\ell(\{v\}) \le \sigma_j$ for every item $v \notin
      J \cup A_{j-1}$ that is feasible at that step.
  \end{enumerate}
\end{lemma}
\begin{proof}
  (i) Expanding definition \Cref{eq:def-residual} twice,
  \begin{equation*}
    f(v \mid U)
    = f(U \cup \{v\}) - f(U)
    = \bigl(\wh G(J \cup U \cup \{v\}) - \wh G(J)\bigr)
      - \bigl(\wh G(J \cup U) - \wh G(J)\bigr)
    = \wh G(v \mid J \cup U),
  \end{equation*}
  which is nonnegative because $\wh G$ is nondecreasing.

  (ii) By construction \Cref{eq:trajectory}, $V$ is affine on each segment
  $[\ell(A_{j-1}), \ell(A_j)]$ with slope $\sigma_j = f(a_j \mid
  A_{j-1})/\ell_{a_j}$, and its value at the right endpoint of segment $j$
  telescopes to
  $V(\ell(A_j)) = f(A_{j-1}) + f(a_j \mid A_{j-1}) = f(A_j)$,
  which equals the value used at the left endpoint of segment $j+1$; hence
  $V$ is continuous. Each slope satisfies $\sigma_j \ge 0$ by (i) and
  $\ell_{a_j} > 0$, so $V$ is nondecreasing. Finally, $\ell(A_m) \ge r$ by
  \cref{lem:coverage}, so $[0, r]$ lies in the domain of $V$.

  (iii) Line 6 of \cref{alg:bps} selects, among all items $i \notin J \cup
  A_{j-1}$ that are feasible at the current state $S = J \cup A_{j-1}$, one
  maximizing $\wh G(i \mid S)/\ell_i$. By (i), $\wh G(v \mid J \cup A_{j-1})
  = f(v \mid A_{j-1})$ for every such $v$, so maximizing the density with
  respect to $\wh G$ is the same as maximizing it with respect to $f$; in
  particular the selected item $a_j$ satisfies
  $f(v \mid A_{j-1})/\ell(\{v\}) \le f(a_j \mid A_{j-1})/\ell_{a_j} =
  \sigma_j$ for every feasible $v$.
\end{proof}

\begin{lemma}[Properties of the bounding function]\label{lem:app-phi}
  The function $\varphi$ of \Cref{eq:phi-def} satisfies:
  \begin{enumerate}
    \item[(i)] $\varphi(0) = 0$; this holds on the first branch when
      $D_1 > 0$ and on the second branch when $D_1 = 0$.
    \item[(ii)] $\varphi$ is continuous on $[0, \infty)$. If
      $0 < D_1 < +\infty$, the two branches agree at $u = D_1$.
    \item[(iii)] On the first branch $(0, D_1)$,
      $\varphi'(u) = \bigl(f(X_1) - \varphi(u)\bigr)/d_1$; on the second
      branch $(D_1, \infty)$,
      $\varphi'(u) = \bigl(f(P) - \varphi(u)\bigr)/D$.
  \end{enumerate}
\end{lemma}
\begin{proof}
  We use throughout the chain-rule decomposition
  \begin{equation}\label{eq:app-fP-decomp}
    f(P) = f(X_1) + f(X_2 \mid X_1) = f(X_1) + \rho_2 d_2 = \rho_1 d_1 + \rho_2 d_2 .
  \end{equation}

  (i) If $D_1 > 0$, the first branch gives $\varphi(0) = f(X_1)(1 - e^0) =
  0$. If $D_1 = 0$, then by definition of $D_1$ we have $\rho_2 > 0$ and
  $d_1 \ln(\rho_1/\rho_2) = 0$, hence $\rho_1 = \rho_2$; the second branch
  then gives, using \Cref{eq:app-fP-decomp},
  $\varphi(0) = f(P) - \rho_2 D e^{0} = \rho_1 d_1 + \rho_2 d_2 - \rho_2(d_1
  + d_2) = 0$.

  (ii) Let $0 < D_1 < +\infty$, so $\rho_2 > 0$ and $e^{-D_1/d_1} =
  \rho_2/\rho_1$. The left limit at $D_1$ along the first branch is
  \begin{equation*}
    f(X_1)\bigl(1 - e^{-D_1/d_1}\bigr)
    = f(X_1)\Bigl(1 - \frac{\rho_2}{\rho_1}\Bigr)
    = f(X_1) - \rho_2 d_1,
  \end{equation*}
  where the last step uses $f(X_1) = \rho_1 d_1$. The value at $D_1$ on the
  second branch is $f(P) - \rho_2 D e^{0} = f(P) - \rho_2(d_1 + d_2)$, which
  equals $f(X_1) - \rho_2 d_1$ by \Cref{eq:app-fP-decomp}. The two branches
  agree, and each branch is continuous, so $\varphi$ is continuous.
  If $D_1 = 0$ or $D_1 = +\infty$, then $\varphi$ is defined on all of
  $[0,\infty)$ by the second or first exponential branch, respectively,
  and is therefore continuous in these cases as well.

  (iii) On $(0, D_1)$, differentiating $\varphi(u) = f(X_1)(1 -
  e^{-u/d_1})$ gives $\varphi'(u) = \frac{f(X_1)}{d_1} e^{-u/d_1}$, while
  $f(X_1) - \varphi(u) = f(X_1) e^{-u/d_1}$; dividing by $d_1$ matches. On
  $(D_1, \infty)$, differentiating $\varphi(u) = f(P) - \rho_2 D
  \exp\bigl(-\frac{u - D_1}{D}\bigr)$ gives $\varphi'(u) = \rho_2
  \exp\bigl(-\frac{u - D_1}{D}\bigr)$, while $f(P) - \varphi(u) = \rho_2 D
  \exp\bigl(-\frac{u - D_1}{D}\bigr)$; dividing by $D$ matches.
\end{proof}

\subsection{\texorpdfstring{Proof of \cref{lem:domination}}{Proof of the Trajectory-Domination Lemma}}
\label{app:proof-domination}

\begin{proof}
  We first bound the density on each segment, then compare $V$ and
  $\varphi$ using differential inequalities and an integrating factor.

  \emph{Step 1 (density lower bound).}
  Fix an accepted segment starting at $A_{j-1}$ with left endpoint $u_0 =
  \ell(A_{j-1}) < r$, and let $\sigma_j = f(a_j \mid A_{j-1})/\ell_{a_j}$ be
  the density of the item selected by \cref{alg:bps} on that segment. Let 
  $Q \subseteq P$ be any nonempty subset such that every member of $Q \setminus A_{j-1}$
  is feasible at this step; by \cref{lem:feasible}, this holds for
  \emph{every} $Q \subseteq P$, since $u_0 < r$. We claim
  \begin{equation}\label{eq:app-density-chain}
    f(Q)
    \;\le\; f(A_{j-1} \cup Q)
    \;\le\; f(A_{j-1}) + \!\!\sum_{v \in Q \setminus A_{j-1}}\!\! f(v \mid A_{j-1})
    \;\le\; f(A_{j-1}) + \sigma_j\, \ell(Q).
  \end{equation}
  The first inequality is monotonicity of $f$. For the second, enumerate $Q
  \setminus A_{j-1} = \{v_1, \dots, v_s\}$ in an arbitrary order and
  telescope:
  \begin{equation*}
    f(A_{j-1} \cup Q) - f(A_{j-1})
    = \sum_{k=1}^{s} f\bigl(v_k \mid A_{j-1} \cup \{v_1, \dots, v_{k-1}\}\bigr)
    \le \sum_{k=1}^{s} f(v_k \mid A_{j-1}),
  \end{equation*}
  where each summand is bounded via submodularity of $f$ (conditioning on a
  superset of $A_{j-1}$ can only decrease the marginal). For the third
  inequality, every $v \in Q \setminus A_{j-1}$ is feasible at this step, so
  \cref{lem:app-residual}(iii) gives $f(v \mid A_{j-1}) \le \sigma_j
  \ell(\{v\})$; summing over $Q \setminus A_{j-1}$ and using $\sigma_j \ge
  0$ together with $\ell(Q \setminus A_{j-1}) \le \ell(Q)$ (lengths are
  positive) yields the claim. Rearranging \Cref{eq:app-density-chain} and
  using $\ell(Q) > 0$,
  \begin{equation}\label{eq:slope-lower}
    \sigma_j \;\ge\; \frac{f(Q) - f(A_{j-1})}{\ell(Q)}
    \;\ge\; \frac{f(Q) - V(u)}{\ell(Q)}
    \qquad \text{for every } u \in [\ell(A_{j-1}), \ell(A_j)],
  \end{equation}
  where the second inequality holds because $V$ is nondecreasing
  (\cref{lem:app-residual}(ii)) with $V(\ell(A_{j-1})) = f(A_{j-1})$, so
  $V(u) \ge f(A_{j-1})$ on the whole segment. We apply
  \Cref{eq:slope-lower} with $Q = X_1$ and $Q = P$ on each segment
  starting before $r$.

  \emph{Step 2 (differential inequality on each branch).}
  Consider the difference $W(u) := V(u) - \varphi(u)$ on $[0, r]$. Both $V$
  and $\varphi$ are continuous (\cref{lem:app-residual}(ii) and
  \cref{lem:app-phi}(ii)), so $W$ is continuous. Partition $[0,
  \min(r, D_1)]$ (and, when $D_1 < r$, also $[D_1, r]$) by the finitely many
  trajectory breakpoints $\ell(A_0) < \ell(A_1) < \cdots$ falling in the
  respective interval. On the interior of each cell of this partition, $V$
  is affine with $V'(u) = \sigma_j$ for the segment index $j$ containing the
  cell, and $\varphi$ is differentiable (\cref{lem:app-phi}(iii)).

  \emph{First branch.} Let $u$ lie in the interior of a cell of $[0,
  \min(r, D_1)]$. The segment containing $u$ has left endpoint $u_0 \le u <
  r$, so \Cref{eq:slope-lower} applies with $Q = X_1$, and
  \cref{lem:app-phi}(iii) gives the ODE for $\varphi$. Subtracting,
  \begin{equation}\label{eq:app-W-branch1}
    W'(u) = V'(u) - \varphi'(u)
    \;\ge\; \frac{f(X_1) - V(u)}{d_1} - \frac{f(X_1) - \varphi(u)}{d_1}
    \;=\; -\frac{W(u)}{d_1}.
  \end{equation}

  \emph{Second branch.} Let $D_1 < r$ and let $u$ lie in the interior of a
  cell of $[D_1, r]$. Again the segment containing $u$ has left endpoint
  $u_0 < r$, so \Cref{eq:slope-lower} applies with $Q = P$, and the same
  subtraction gives
  \begin{equation}\label{eq:app-W-branch2}
    W'(u) \;\ge\; \frac{f(P) - V(u)}{D} - \frac{f(P) - \varphi(u)}{D}
    \;=\; -\frac{W(u)}{D}.
  \end{equation}

  \emph{Step 3 (integrating factor and propagation across breakpoints).}
  Fix a branch and write $\ell_Q$ for its length scale ($\ell_Q = d_1$ on the first branch, $\ell_Q = D$ on the second). On the interior of each cell,
  \begin{equation}\label{eq:app-integrating-factor}
    \frac{\mathrm{d}}{\mathrm{d}u}\Bigl(e^{u/\ell_Q}\, W(u)\Bigr)
    = e^{u/\ell_Q}\Bigl(W'(u) + \frac{W(u)}{\ell_Q}\Bigr) \;\ge\; 0
  \end{equation}
  by \Cref{eq:app-W-branch1} or \Cref{eq:app-W-branch2}. Hence
  $e^{u/\ell_Q} W(u)$ is nondecreasing on the interior of each cell; since it
  is continuous on the whole branch interval and the partition has finitely
  many cells, it is nondecreasing on the entire branch interval. Therefore,
  if $W \ge 0$ at the left endpoint of the branch interval, then
  $e^{u/\ell_Q} W(u) \ge 0$, hence $W(u) \ge 0$, throughout that interval.

  By \Cref{eq:trajectory} and \cref{lem:app-phi}(i), $W(0) = 0$.
  If $D_1 \ge r$, including $D_1 = +\infty$ when $\rho_2 = 0$, the
  first-branch argument gives $W \ge 0$ on $[0,r]$; continuity includes
  the endpoint when $D_1 = r$. If $D_1 = 0$, the second-branch argument
  applies from $u = 0$. Finally, if $0 < D_1 < r$, the first-branch
  argument and continuity give $W(D_1) \ge 0$, which provides the initial
  value for the second branch. Thus $V(u) \ge \varphi(u)$ for all
  $u \in [0,r]$, proving \cref{lem:domination}.
\end{proof}

\subsection{\texorpdfstring{Proof of Lemma~\ref{lem:static}}{Proof of the Static-Bound Lemma}}
\label{app:proof-static}

We use the following standard inequality; we include its one-line proof for
completeness.

\begin{lemma}[Log-sum inequality]\label{lem:app-logsum}
  For positive reals $a_1, a_2, b_1, b_2$,
  \begin{equation}\label{eq:app-logsum}
    a_1 \ln\frac{b_1}{a_1} + a_2 \ln\frac{b_2}{a_2}
    \;\le\;
    (a_1 + a_2) \ln\frac{b_1 + b_2}{a_1 + a_2}.
  \end{equation}
\end{lemma}
\begin{proof}
  Apply Jensen's inequality to the concave function $\ln$ with weights
  $\lambda_i = a_i/(a_1 + a_2)$ and points $u_i = b_i/a_i$:
  \begin{equation*}
    \frac{a_1}{a_1 + a_2}\ln\frac{b_1}{a_1} + \frac{a_2}{a_1 + a_2}\ln\frac{b_2}{a_2}
    \;\le\;
    \ln\Bigl(\frac{a_1}{a_1+a_2}\cdot\frac{b_1}{a_1} + \frac{a_2}{a_1+a_2}\cdot\frac{b_2}{a_2}\Bigr)
    = \ln\frac{b_1 + b_2}{a_1 + a_2}.
  \end{equation*}
  Multiplying both sides by $a_1 + a_2 > 0$ gives \Cref{eq:app-logsum}.
\end{proof}

\begin{proof}[Proof of Lemma~\ref{lem:static}]
  Abbreviate $C := \wh G(J)$ and $H := \wh G(T)$ for the duration of this
  proof. Since $J \cup P = T$, definition \Cref{eq:def-residual} gives
  \begin{equation}\label{eq:app-H-decomp}
    f(P) = \wh G(T) - \wh G(J) = H - C .
  \end{equation}

  \emph{Step 0 (reduction and common identities).}
  If $f(P) = 0$, then $f(X_1) = 0$ by nonnegativity and monotonicity of $f$.
  Thus $\rho_2 = 0$ and $D_1 = +\infty$, so \Cref{eq:phi-def} gives
  $\varphi \equiv 0$. Together with \Cref{eq:app-H-decomp}, this yields
  \begin{equation*}
    C + \varphi(r) = C = H \ge \alpha H .
  \end{equation*}
  Assume from now on that $f(P) > 0$. Then $\rho_1 > 0$: otherwise $f(X_1) =
  0$, and the block order \Cref{eq:block-order} forces $f(X_2) \le
  \frac{d_2}{d_1} f(X_1) = 0$, whence by submodularity $f(X_2 \mid X_1) \le
  f(X_2) = 0$ and $f(P) = f(X_1) + f(X_2 \mid X_1) \le 0$, a contradiction.

  Define $(x, z)$ according to the block order:
  \begin{equation}\label{eq:app-xz-def}
    (x, z) =
    \begin{cases}
      \bigl(f(\{v^\star\}),\; f(R \mid \{v^\star\})\bigr), & X_1 = \{v^\star\},\\[0.3em]
      \bigl(f(v^\star \mid R),\; f(R)\bigr), & X_1 = R.
    \end{cases}
  \end{equation}
  In both orders the chain rule gives $x + z = f(P)$, so by
  \Cref{eq:app-H-decomp},
  \begin{equation}\label{eq:app-H-xz}
    H = C + x + z .
  \end{equation}
  In both orders we also have
  \begin{equation}\label{eq:app-C2x}
    C \ge 2x .
  \end{equation}
  Indeed, if $X_1 = \{v^\star\}$ then $x = f(\{v^\star\})$ and
  \Cref{eq:seed-two-marginals} applies directly with $v = v^\star$; if $X_1
  = R$ then submodularity gives $x = f(v^\star \mid R) \le
  f(\{v^\star\})$, and \Cref{eq:seed-two-marginals} gives $C \ge
  2f(\{v^\star\}) \ge 2x$.

  Finally, recall the block data in the two orders. If $X_1 =
  \{v^\star\}$: $d_1 = p$, $d_2 = r$, $\rho_1 = x/p$, $\rho_2 = z/r$. If
  $X_1 = R$: $d_1 = r$, $d_2 = p$, $\rho_1 = z/r$, $\rho_2 = x/p$. In both
  orders, whenever $x, z > 0$,
  \begin{equation}\label{eq:app-weighted-log}
    d_1 \ln\rho_1 + d_2 \ln\rho_2
    \;=\; p \ln\frac{x}{p} + r \ln\frac{z}{r},
  \end{equation}
  because the two summands on the left are exactly the two summands on the
  right, possibly in the opposite order.

  We now distinguish the same three structural cases as the refined
  analysis of \cite{kulik2021refined}; Step 4 verifies that they are
  exhaustive.

  \emph{Step 1 (Case 1: $r \ge D_1$).}
  Since $r < \infty$, the case condition forces $D_1 < \infty$, hence
  $\rho_2 > 0$ by the definition of $D_1$, and with $\rho_1 \ge \rho_2$ both
  block values are positive; reading off the block data above, this means
  \begin{equation*}
    x > 0 \quad\text{and}\quad z > 0
  \end{equation*}
  in both orders, so every logarithm below has a strictly positive
  argument. Because $r \ge D_1$, the second branch of \Cref{eq:phi-def}
  evaluates $\varphi(r)$, and by \Cref{eq:app-H-decomp} and
  \Cref{eq:app-H-xz},
  \begin{equation}\label{eq:app-case1-start}
    C + \varphi(r)
    = C + f(P) - \rho_2 D \exp\Bigl(-\frac{r - D_1}{D}\Bigr)
    = H - \rho_2 D \exp\Bigl(-\frac{r - D_1}{D}\Bigr).
  \end{equation}
  It therefore suffices to bound the exponential tail by $H/e$.

  First, combining \Cref{eq:app-C2x} with \Cref{eq:app-H-xz},
  \begin{equation}\label{eq:app-3x}
    H - z = C + x \ge 2x + x = 3x .
  \end{equation}
  Next we bound the weighted-log expression \Cref{eq:app-weighted-log}.
  Using $x \le (H - z)/3$ from \Cref{eq:app-3x} and the monotonicity of
  $\ln$, then splitting off $\ln 3$, then applying
  \cref{lem:app-logsum} with $(a_1, a_2) = (p, r)$ and $(b_1, b_2) = (H - z,
  z)$ (all four are positive, $p + r = D$, and $(H - z) + z = H$), and
  finally using $\ln 3 > 1$ with $p > 0$:
  \begin{equation}\label{eq:app-logsum-chain}
    \begin{aligned}
      p \ln\frac{x}{p} + r \ln\frac{z}{r}
      &\le p \ln\frac{H - z}{3p} + r \ln\frac{z}{r} \\
      &= -\,p \ln 3 + p \ln\frac{H - z}{p} + r \ln\frac{z}{r} \\
      &\le -\,p \ln 3 + D \ln\frac{H}{D} \\
      &\le -\,p + D \ln\frac{H}{D}.
    \end{aligned}
  \end{equation}
  Now rewrite the exponential tail. Using $r = D - p$,
  \begin{equation*}
    \rho_2 D \exp\Bigl(-\frac{r - D_1}{D}\Bigr)
    = D \exp\Bigl(\ln\rho_2 - \frac{D - p - D_1}{D}\Bigr)
    = D \exp\Bigl(-1 + \frac{p + D_1}{D} + \ln\rho_2\Bigr),
  \end{equation*}
  and since $D_1 = d_1 \ln(\rho_1/\rho_2)$ and $D = d_1 + d_2$,
  \begin{equation*}
    \frac{D_1}{D} + \ln\rho_2
    = \frac{d_1 \ln\rho_1 - d_1 \ln\rho_2 + D \ln\rho_2}{D}
    = \frac{d_1 \ln\rho_1 + d_2 \ln\rho_2}{D},
  \end{equation*}
  so that, by \Cref{eq:app-weighted-log},
  \begin{equation}\label{eq:app-tail-rewrite}
    \rho_2 D \exp\Bigl(-\frac{r - D_1}{D}\Bigr)
    = D \exp\Bigl(-1 + \frac{p + p\ln\frac{x}{p} + r\ln\frac{z}{r}}{D}\Bigr).
  \end{equation}
  Substituting the bound \Cref{eq:app-logsum-chain} into
  \Cref{eq:app-tail-rewrite} and using the monotonicity of $\exp$,
  \begin{equation}\label{eq:app-tail-bound}
    \rho_2 D \exp\Bigl(-\frac{r - D_1}{D}\Bigr)
    \;\le\; D \exp\Bigl(-1 + \frac{D \ln(H/D)}{D}\Bigr)
    = D \cdot \frac{1}{e} \cdot \frac{H}{D}
    = \frac{H}{e}.
  \end{equation}
  Combining \Cref{eq:app-case1-start} and \Cref{eq:app-tail-bound},
  \begin{equation*}
    C + \varphi(r) \;\ge\; H - \frac{H}{e} \;=\; \alpha H .
  \end{equation*}
  The boundary subcases are included: $r = D_1$ uses the second branch of
  \Cref{eq:phi-def}, which is defined for all $u \ge D_1$, and $D_1 = 0$
  poses no difficulty since only the case condition $r \ge D_1$ was used.

  \emph{Step 2 (Case 2: $r < D_1$ and $X_1 = \{v^\star\}$).}
  Here $x = f(\{v^\star\})$, $z = f(R \mid \{v^\star\})$, $d_1 = p$, and we
  write
  \begin{equation}\label{eq:app-delta}
    \delta := \frac{r}{p} > 0 .
  \end{equation}
  Note $x = \rho_1 p > 0$ because $\rho_1 > 0$ (Step 0). Since $r < D_1$,
  the first branch of \Cref{eq:phi-def} evaluates $\varphi(r)$:
  \begin{equation}\label{eq:app-case2-phi}
    \varphi(r) = x\bigl(1 - e^{-r/p}\bigr) = x\bigl(1 - e^{-\delta}\bigr).
  \end{equation}

  Suppose first $z > 0$. Then $\rho_2 = z/r > 0$ and the case condition $r
  < D_1 = p \ln\frac{x/p}{z/r}$ can be divided by $p > 0$ to read
  \begin{equation*}
    \delta \;<\; \ln\Bigl(\frac{x}{p} \cdot \frac{r}{z}\Bigr)
    \;=\; \ln\Bigl(\frac{x}{z}\,\delta\Bigr).
  \end{equation*}
  Exponentiating gives
  $e^{\delta} < \frac{x}{z}\delta$, i.e.,
  \begin{equation}\label{eq:app-case2-x-large}
    x \;>\; z\,\frac{e^{\delta}}{\delta} .
  \end{equation}
  Combining this bound with $C \ge 2x$ from \Cref{eq:app-C2x} gives
  \begin{equation}\label{eq:app-case2-chain}
    \begin{aligned}
      C + \varphi(r)
      &= C + x\bigl(1 - e^{-\delta}\bigr) \\
      &= \tfrac{2}{3}(C + x) + \Bigl[\tfrac{1}{3}(C + x) - x e^{-\delta}\Bigr] \\
      &\ge \tfrac{2}{3}(C + x) + x\bigl(1 - e^{-\delta}\bigr) \\
      &> \tfrac{2}{3}(C + x) + z\,\frac{e^{\delta} - 1}{\delta} \\
      &\ge \tfrac{2}{3}(C + x) + z \\
      &\ge \tfrac{2}{3}(C + x + z)
      = \tfrac{2}{3} H
      \;>\; \alpha H .
    \end{aligned}
  \end{equation}
  Here $C \ge 2x$ gives $(C+x)/3 \ge x$, and
  \Cref{eq:app-case2-x-large} bounds the remaining term
  $x(1-e^{-\delta})$ below by $z(e^\delta-1)/\delta$.
  The subsequent inequalities use $e^\delta \ge 1+\delta$,
  $H=C+x+z$, and $2/3 > \alpha$.

  If $z = 0$, then $H = C+x$. Using $C \ge 2x$ and
  $\varphi(r) \ge 0$ directly gives
  \begin{equation*}
    C + \varphi(r)
    \;\ge\; C
    \;\ge\; \tfrac{2}{3}(C + x)
    \;=\; \tfrac{2}{3} H
    \;>\; \alpha H.
  \end{equation*}

  \emph{Step 3 (Case 3: $r < D_1$ and $X_1 = R$).}
  Here $z = f(R)$, $x = f(v^\star \mid R)$, and $d_1 = \ell(R) = r$. Since
  $r < D_1$, the first branch of \Cref{eq:phi-def} evaluates $\varphi(r)$,
  and because the branch length scale is $d_1 = r$,
  \begin{equation}\label{eq:app-case3-phi}
    \varphi(r) = f(X_1)\bigl(1 - e^{-r/d_1}\bigr) = z\bigl(1 - e^{-1}\bigr) = \alpha z .
  \end{equation}
  From \Cref{eq:app-C2x}, $C \ge 2x$, hence $3C \ge 2C + 2x$, i.e.,
  \begin{equation}\label{eq:app-case3-C}
    C \;\ge\; \tfrac{2}{3}(C + x) \;\ge\; \alpha (C + x),
  \end{equation}
  where the second inequality again uses $\tfrac{2}{3} > \alpha$ (from $e <
  3$) and $C + x \ge 0$. Adding \Cref{eq:app-case3-phi} and
  \Cref{eq:app-case3-C} and using \Cref{eq:app-H-xz},
  \begin{equation*}
    C + \varphi(r) \;\ge\; \alpha(C + x) + \alpha z \;=\; \alpha(C + x + z) \;=\; \alpha H .
  \end{equation*}

  \emph{Step 4.}
  The reduction $f(P) = 0$ was handled in Step 0, so assume $f(P) > 0$. If
  $D_1 \le r$ — which includes $D_1 = 0$, since $r > 0$ by
  \Cref{eq:app-positive-lengths} — we are in Case 1. If $r < D_1$ —
  which includes $D_1 = +\infty$, i.e., $\rho_2 = 0$ — we are in Case 2 or
  Case 3 according to whether the block order puts $\{v^\star\}$ or $R$
  first; these two options are exhaustive because $(X_1, X_2)$ is a
  permutation of $\bigl(\{v^\star\}, R\bigr)$. In every case $\wh G(J) +
  \varphi(r) \ge \alpha \wh G(T)$, which is the claim of \cref{lem:static}.
\end{proof}

\section{Experimental Details}
\label{app:details}

\subsection{Executor and execution protocol}
\label{app:executor}

We use Qwen3-32B as the frozen executor throughout all experiments, served with vLLM in FP8. Each rollout consists of a single chat completion at temperature $0.7$, with thinking disabled and an output limit of $4{,}096$ tokens. The executor has no tools and is not shown the tests. We adapt each task's BigCodeBench prompt to refer to the private modules and prepend the selected skill documents to the user message in a fixed order, separated by horizontal rules. We tokenize each skill document with the Qwen3 tokenizer and define the token cost of a selected set as the sum of its document lengths, excluding separators and the task prompt.

A rollout succeeds if the code extracted from the completion passes the task's full test suite in the BigCodeBench evaluation harness, run in an isolated process with a $240$-second time limit. For each task--set pair, the success rate is the fraction of rollouts that pass all tests. The evaluation process loads the private modules from compiled bytecode with docstrings removed; their source files are not exposed to the generated program.

Offline data collection starts by enumerating every non-empty subset of each task-specific candidate skill pool, with four rollouts per distinct task--set pair. Selected test-task pairs receive one or two additional batches of $12$ rollouts, giving $16$ or $28$ rollouts in total. Subsequent collections for newly added tasks and additional skill combinations use eight rollouts per pair. Pooling outcomes by task--set pair yields $5{,}204$ pairs and $45{,}400$ rollouts for the fitting and evaluation protocols in \cref{app:fit,app:validity}.

During online training (\cref{app:e2e}), the learner receives one execution outcome for each selected task--set pair. It uses cached outcomes without replacement and collects additional rollouts when needed.

For the final end-to-end evaluation, we collect $16$ fresh rollouts
for each distinct task--set pair selected by any method.
These outcomes are used only for evaluation, and
\cref{fig:selection} reports results from this evaluation.

In total, the reported results rest on $90{,}588$ distinct rollouts. They comprise the $45{,}400$ offline rollouts above, the $100$ no-skill rollouts of \cref{app:testbed}, $43{,}088$ further rollouts that supply the online-training feedback of \cref{app:e2e}, $20{,}228$ collected before training and $22{,}860$ during it, and the $2{,}000$ rollouts of the final end-to-end evaluation.

\subsection{Testbed}
\label{app:testbed}

We replace \texttt{base64}, \texttt{binascii}, \texttt{hashlib}, \texttt{hmac}, and \texttt{zlib}/\texttt{gzip} with five private modules: \texttt{pk64}, \texttt{pkbin}, \texttt{pkdigest}, \texttt{pkmac}, and \texttt{pkzip}, respectively. For each operation, we provide three function variants with different encodings or constants. We randomly designate one variant as canonical and keep this assignment fixed across the tasks, tests, and skill documents. The canonical variants define the behavior expected by the task tests. The library includes documentation of these variants and distractor documents describing non-canonical variants of the same operations.

A modified task passes our screening if its reference solution using the private modules passes all tests, while the original standard-library solution and every hybrid solution fail. Each hybrid replaces the calls to one required private module with the corresponding standard-library operations, while retaining the other private-module calls. Of the $24$ tasks that pass this screening, we exclude two with prompt ambiguities unrelated to the private modules and one in which a required private-module operation can be reproduced without its documentation. The final testbed contains $21$ tasks: $20$ require two capabilities and one requires a single capability. Without injected skills, the executor records $0$ successes in $100$ rollouts on the modified tasks, compared with $58$ successes in $68$ rollouts on the original BigCodeBench tasks under the same execution protocol.

The skill library contains $31$ documents: $17$ cover a single module, $10$ cover pairs of modules, and $4$ are distractors. Each of \texttt{pk64}, \texttt{pkbin}, \texttt{pkdigest}, and \texttt{pkzip} has four documents---a cheat sheet, a quick reference card, a guide, and an FAQ---while \texttt{pkmac} has one reference card. Document lengths range from $206$ to $1{,}302$ tokens under the tokenizer specified in \cref{app:executor}.

For the value-model experiments in \cref{sec:evaluation:e1}, we divide the $21$ tasks into $11$ training tasks and $10$ test tasks. The end-to-end experiment uses eight of these test tasks, with $24$ instances in \cref{sec:evaluation:e2}. The other two, \texttt{00505} and \texttt{00934}, are among the four tasks added so that every module's supply is identifiable in \cref{fig:recovery}: \texttt{pkmac}, and \texttt{pkzip} apart from \texttt{pk64}, are demanded only by them. The optimization instances of \cref{app:opt} are assembled from all $21$. The remaining $13$ tasks are used to warm up the neural capability encoder and the supervised retriever (\cref{app:e2e}). The example in \cref{fig:moti} uses task \texttt{00028}, which requires \texttt{pk64} and \texttt{pkzip}.

\subsection{Value-model parameterization}
\label{app:parameterization}

For the value-model experiments in \cref{sec:evaluation:e1}, we instantiate the structured objective of \cref{sec:model} with $d=5$ latent capability dimensions, matching the number of private-module families in the testbed. The task set and skill library are fixed, so the capability encoders are lookup tables: one supply vector $\wh\vu_i\in\mathbb R_+^5$ per skill and one demand vector $\wh\vw^q\in\mathbb R_+^5$ per task. These vectors estimate the effective supplies and demands in \cref{eq:def-effective}, which already absorb the executor-specific factors $\lambda_k^E$ and $\eta_k^E$. We learn the vectors from execution outcomes, without using task or document text or capability annotations. No latent dimension is assigned to a particular module during fitting; the dimensions are matched to modules only when evaluating parameter recovery in \cref{fig:recovery}.

The fitted objective in \cref{eq:def-hat-F} then takes the form
\begin{equation}
  \widehat F_E(q,S)
  = \sum_{k=1}^{5}\wh w_k^q\,
    h_k\!\left(\sum_{i\in S}\wh u_{i,k}\right)
    -\wh\kappa_E\ell(S).
  \label{eq:app-value-model}
\end{equation}
The executor $E$ and the document lengths $\ell_i$ are fixed as described in \cref{app:executor}. We learn the context sensitivity $\wh\kappa_E\ge 0$ jointly with the supply and demand vectors, whose entries are kept nonnegative through a softplus.

We use the same fixed response function in every dimension:
\begin{equation}
  h_k(x)
  = \log\left(\epsilon_0+(1-\epsilon_0)(1-e^{-x})\right) - \log{\epsilon_0},
  \qquad x\ge 0,\qquad \epsilon_0=0.01.
  \label{eq:app-response}
\end{equation}
This function is increasing and concave, with $h_k(0)=0$ and an upper bound of $\log(1/\epsilon_0)$. It assigns diminishing gains as supply accumulates in a capability dimension and satisfies the conditions in \cref{sec:model:objective}. Neither its functional form nor $\epsilon_0$ is learned.

Since $\widehat F_E(q,\varnothing)=0$, predicting an absolute execution score also requires a task-specific offset $\wh b_q$. We parameterize it as
\begin{equation}
  \wh b_q=-a_q-\sum_{k=1}^{5}\wh w_k^q\log(1/\epsilon_0),
  \qquad a_q\ge 0.
  \label{eq:app-offset}
\end{equation}
This adds one learnable scalar $a_q$ per task; $\wh b_q$ is determined by that scalar and the task's demand vector. The bound on $h_k$ ensures that $\wh b_q+\widehat F_E(q,S)\le 0$, allowing us to convert this score into a probability as described next. With $31$ skills and $21$ tasks, the model has $31\times5+21\times5+21+1=282$ learnable parameters.

\subsection{Execution success and model fitting}
\label{app:fit}

\textbf{From success probability to execution score.}
\Cref{sec:model} allows different execution scores; here we consider task success probability. Let $p_E(q,S)$ be the probability that the frozen executor passes all tests for task $q$ when given the skills in $S$, under the execution protocol of \cref{app:executor}. Our goal is to select a feasible set with the highest success probability.

When performance is measured by task success probability, a task's requirements must be considered jointly. In our testbed, a task may use several private-module APIs, and failure in any required operation can cause the entire execution to fail. This motivates modeling success probability as a product of factors associated with the required capabilities. Taking the logarithm turns this product into a sum of capability contributions, consistent with the demand-weighted benefit in \cref{sec:model:objective}. We therefore define the execution score on a log-probability scale in these experiments.

To keep the score finite even when the success probability is zero, we first apply the same smoothing used for the execution labels:
\begin{equation}
  p_{E,\epsilon}(q,S)=\epsilon_0+(1-\epsilon_0)p_E(q,S),
  \qquad J_E(q,S)=\log p_{E,\epsilon}(q,S).
  \label{eq:app-execution-score}
\end{equation}
Here $J_E$ is the execution score introduced in \cref{sec:model}, and $\epsilon_0=0.01$ is the fixed constant in \cref{eq:app-response}. Both smoothing and the logarithm are strictly increasing, so for each task,
\begin{equation}
  \operatorname*{arg\,max}_{S\in\gF_B}p_E(q,S)
  =\operatorname*{arg\,max}_{S\in\gF_B}J_E(q,S)
  =\operatorname*{arg\,max}_{S\in\gF_B}F_E^\star(q,S).
  \label{eq:app-score-equivalence}
\end{equation}
The last equality follows from $F_E^\star(q,S)=J_E(q,S)-J_E(q,\varnothing)$. Thus, the score transformation preserves the original selection objective. 

\textbf{Predicting success and fitting the parameters.}
The structured objective predicts a score improvement. Adding the task offset estimates the absolute score, giving the smoothed probability prediction
\begin{equation}
  \widehat p_{E,\epsilon}(q,S)
  =\exp\!\left(\wh b_q+\widehat F_E(q,S)\right).
  \label{eq:app-probability-prediction}
\end{equation}
The offset parameterization in \cref{eq:app-offset} ensures that this prediction never exceeds one. The offset $\wh b_q$ represents the model's no-skill score because $\log\widehat p_{E,\epsilon}(q,\varnothing)=\wh b_q$. For a fixed task, the offset is constant across sets, so maximizing $\widehat F_E$ also maximizes the predicted probability.

Each rollout provides an outcome $y\in\{0,1\}$ indicating whether all tests passed. We smooth it to $t=\epsilon_0+(1-\epsilon_0)y$ and fit the model by minimizing binary cross-entropy,
\begin{equation}
  \mathcal L(t,\widehat p)
  =-t\log\widehat p-(1-t)\log(1-\widehat p),
  \qquad \widehat p=\widehat p_{E,\epsilon}(q,S),
  \label{eq:app-fitting-loss}
\end{equation}
averaged over the training rollouts. The conditional mean of $t$ is $p_{E,\epsilon}(q,S)$, so the smoothed labels and probability prediction use the same target. We optimize the parameters jointly with full-batch Adam for $4{,}000$ steps, using an initial learning rate of $0.05$ and a cosine schedule. We run three random restarts and retain the fit with the lowest training loss. The offline data contain the $5{,}204$ measured task--set pairs of \cref{app:executor}, with pooled success and failure counts; the protocols in \cref{app:validity} specify which records are used for fitting. The no-skill executions are excluded from fitting and are used only to evaluate the empty-set choice.

\subsection{Value-model validity}
\label{app:validity}

We provide implementation and evaluation details for \cref{sec:evaluation:e1} and supplement its prediction results with two experiments on generalization to held-out skill combinations.

\textbf{Comparison models.}
For the value-model comparison in \cref{sec:evaluation:e1}, every model receives the task identity and a binary membership vector over the $31$ documents; none uses text. Additive fit learns one coefficient per task and document plus a task offset ($672$ parameters). Shapley scores use the same additive form, with each document's value set to its mean marginal contribution over all training pairs $(S,S\cup\{i\})$ ($651$). Datamodels fit one ridge regression per task to predict the smoothed success rate from the membership vector, with the ridge strength chosen on a held-out fifth of that task's training sets ($672$). DeepSets embeds each document in $32$ dimensions, sums the embedded members through an MLP, and concatenates a learned task embedding before the output MLP ($16{,}929$); Set Transformer replaces the sum with one four-head self-attention block and attention pooling ($16{,}577$). Linear response replaces the saturating function in our model with $h_k(x)=x$. The neural models use the same smoothed binary cross-entropy, restarts, and step budget as our model in \cref{app:fit}, with weight decay $10^{-4}$; Datamodels and Shapley use least squares and averaged marginal contributions, respectively.

\textbf{Main prediction experiment.}
Ranking accuracy measures how often a model correctly orders two sets from the same task whose measured success rates differ by more than $0.25$; a predicted tie receives half credit. For \cref{fig:models}, every model is fitted on all measured sets of the $11$ training tasks and on $20$ sets of each test task, sampled once with a fixed seed. Evaluation uses the remaining $2{,}806$ sets of the test tasks. Linear response predicts a rate of one on every held-out set, so all of its comparisons count as ties.

\textbf{Main selection experiment.}
For the selection panel of \cref{fig:models}, we randomly split each task--set pair's rollouts into two halves, fit every model on one half, and evaluate its selected sets on the other. Fitting uses every measured set of the $11$ training tasks and $20$ measured sets of each test task. The $204$ selection instances use candidate pools of four or five documents over eight test tasks, with budgets from $689$ to $1{,}593$ tokens. Each model enumerates every non-empty feasible subset and selects the one it scores highest. The empirical ceiling averages the highest measured success rate among these subsets in each instance, using the evaluation half. We repeat the split with three random seeds and report the mean success rate. On every seed, the sets our model selects succeed more often than those of DeepSets and Set Transformer, by $0.06$--$0.12$ and $0.03$--$0.10$ respectively.

\begin{figure}[t]
\centering
\includegraphics[width=0.7\linewidth]{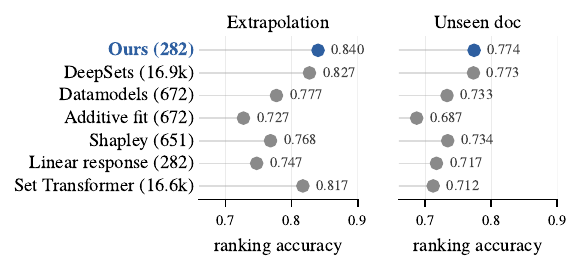}
\caption{Supplementary prediction experiments for \cref{sec:evaluation:e1}, reporting pairwise ranking accuracy with parameter counts in parentheses. \emph{Left (extrapolation):} training uses sets of one or two documents; testing uses sets of three or more. \emph{Right (unseen doc):} each fold withholds all combinations containing one target document, retains its singleton records in training, and tests on the withheld combinations.}
\label{fig:protocols}
\end{figure}

\textbf{Supplementary prediction experiments.}
The main prediction experiment evaluates held-out sets after fitting on a sample of sets from each test task. The two experiments in \cref{fig:protocols} impose specific restrictions on the combinations available during fitting: one withholds all larger sets, and the other withholds all combinations containing a particular document. Each protocol repartitions the full $5{,}204$ measured task--set pairs by set composition. All models use the same partitions within each protocol and are evaluated by the pairwise ranking accuracy defined above. Both experiments examine generalization to new combinations within the fixed task set.

\textbf{Extrapolation: from small sets to larger combinations.}
This experiment asks whether a model can predict the relative value of larger skill sets after learning only from singletons and pairs. We use a single split: the training set contains the $2{,}099$ measured task--set pairs with one or two documents, and the test set contains the remaining $3{,}105$ pairs with three or more documents. No execution outcomes from the larger sets are used for fitting. Our model achieves a ranking accuracy of $0.840$, compared with $0.827$ for DeepSets (\cref{fig:protocols}, left). This result supports the fitted objective's ability to rank skill sets beyond the sizes observed during training on this testbed.

\textbf{Unseen doc: predicting new combinations of an observed skill.}
This experiment asks whether a model can rank combinations containing a document when that document has been observed only on its own. We use $17$ folds, one for each single-module document. In a given fold, the training set contains all measured task--set pairs whose set excludes the target document, together with all pairs whose set consists of that document alone. The test set contains all measured pairs in which the target document appears with at least one other document. Thus, its singleton outcomes remain available for fitting, while every combination containing it is held out. Our model achieves a ranking accuracy of $0.774$, close to DeepSets at $0.773$ (\cref{fig:protocols}, right) with $60\times$ fewer parameters. The result supports predicting new combinations of a known skill from its singleton observations and the observed combinations of other skills.

\textbf{Sensitivity to smoothing.}
With $\epsilon_0=0.05$ in place of $0.01$, the share of instances on which \BPS{} reaches the reference in \cref{app:opt} changes by at most $0.013$ at any library size, and \BPS{} remains the top method at every size.

\subsection{Optimization quality}
\label{app:opt}

We evaluate how well \BPS{} solves the fitted skill-selection problem in \Cref{eq:selection-fitted} as the library grows. The instances are built from the fitted parameters, with libraries of $31$ to $124$ skills and selected sets of up to twenty. Within each instance, all methods use the same objective and token budget, and we evaluate their selected sets under that objective. This comparison isolates optimization quality from the model's ability to predict execution success.

\textbf{Libraries and instances.}
We build libraries larger than the testbed's by tiling its fitted parameters. The parameters come from the structured model of \cref{app:parameterization} fitted on all $5{,}204$ measured task--set pairs, the fit shown in \cref{fig:recovery}, shifted so that $\varnothing$ scores zero: a supply vector and a length for each of the $31$ skills and a demand vector for each of the $21$ tasks, all over five capability dimensions. \Cref{fig:tiled} shows a library with three families: each family is a copy of the fitted library with its own five capability dimensions, its supplies and lengths perturbed by log-normal noise with scale $0.2$ and $0.1$. A task is assembled from one fitted task per family, drawn at random from the $21$. It places each drawn task's demand on that family's dimensions, so covering it takes skills from every family. We take one, two and four families, giving $L=31$, $62$ and $124$ skills, and assemble $20$ tasks for each library size. For each task we find the cheapest set of skills that covers its capabilities and set the token budget $B$ to $1$, $1.5$, $2$, $3$ and $5$ times that cost. The instances add to the objective a per-token price $\rho/B$ with $\rho\in\{0,1,2\}$, the settings of \cref{sec:evaluation:e2}, giving $300$ instances per library size. The sets \BPS{} selects hold two skills on average at $L=31$ and twelve at $L=124$.

\begin{figure}[!ht]
\centering
\includegraphics[width=\linewidth]{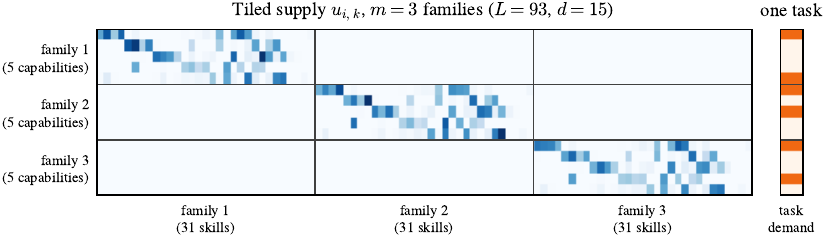}
\caption{A tiled library with three families. \emph{Left:} the supply matrix, $15$ capabilities by $93$ skills; each family is a perturbed copy of the fitted library in \cref{fig:recovery} over its own five capability dimensions. \emph{Right:} one task's demand, one fitted task per family.}
\label{fig:tiled}
\end{figure}

\textbf{Baselines.}
For a fixed instance, write $f(S)$ for its fitted objective, including the added per-token price. Every method respects the token budget.
\begin{itemize}[leftmargin=*,nosep]
    \item \textbf{Density greedy} starts from the empty set and repeatedly adds the feasible skill maximizing $[f(S\cup\{i\})-f(S)]/\ell_i$ until no further skill fits. It returns the better of its final set and $\varnothing$ under $f$, without retaining intermediate prefixes.
    \item \textbf{Gear}~\cite{guo2026budgetedprofit} follows the density rule of cost-benefit greedy~\cite{khuller1999budgeted} and stops once no remaining skill increases $f$. It returns the better of this set and the best single skill.
    \item \textbf{Greedy+Max}~\cite{yaroslavtsev2020bring} follows the density-greedy chain and, at each step, also forms the candidate that augments the current set with the feasible skill of largest marginal gain $f(S\cup\{i\})-f(S)$. It returns the candidate with the highest value of $f$. The original analysis assumes a monotone objective, so we stop the chain once no feasible skill increases $f$.
    \item \textbf{Twin Greedy}~\cite{gong2024budget} grows two disjoint sets in parallel, each time adding the skill and set with the largest ratio of unpriced marginal gain, less twice the price of the skill, to that price. It restarts from every feasible seed of at most two skills and returns the best set found.
    \item \textbf{Threshold Twin Greedy}~\cite{gong2024budget} is its faster variant, which replaces each exact maximization with a decreasing threshold.
    \item \textbf{Best-of-$200$ random} generates $200$ random feasible candidates, each filled until no further skill fits, and returns the candidate with the highest value of $f$.
\end{itemize}

\begin{figure}[!ht]
\centering
\includegraphics[width=0.62\linewidth]{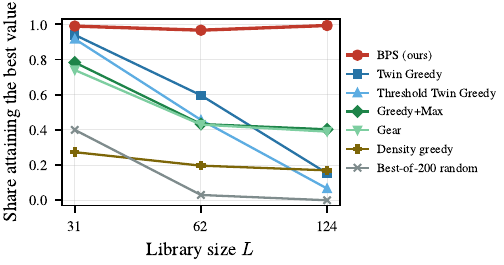}
\caption{Optimization quality as the library grows. Each point is the share of instances on which a method's selected set reaches the best value found by any method, over $300$ instances per library size.}
\label{fig:optquality}
\end{figure}

\textbf{Results.}
Enumeration is out of reach at these sizes, so on each instance we take the highest objective value any compared method finds as the reference, and a method attains it when its selected set's value is within $10^{-9}$ of that value. The reference is close to the optimum: relaxing the selection to fractional skills gives a concave problem whose Frank--Wolfe solution upper-bounds the optimum, and this bound exceeds the reference by only $0.5\%$ on average at $L=31$ and $0.1\%$ at $L=124$. \Cref{fig:optquality} reports how often each method attains the reference. \BPS{} does so on $99\%$, $97\%$ and $99\%$ of the instances at $L=31$, $62$ and $124$; on the remaining $16$ instances Twin Greedy finds a set whose value is higher by less than $0.02\%$. Averaged over the three sizes, \BPS{} attains the reference on $98\%$ of the instances and Twin Greedy on $56\%$, as quoted in \cref{sec:introduction,sec:guarantee}. Twin Greedy, the strongest baseline at $L=31$ with $94\%$, falls to $15\%$ at $L=124$ at a comparable number of objective evaluations, and Threshold Twin Greedy falls from $92\%$ to $7\%$. Greedy+Max and Gear fall from $78\%$ and $74\%$ to $40\%$ and $39\%$, density greedy stays below $30\%$, and best-of-$200$ random drops from $40\%$ to none. We also remove one part of \BPS{} at a time. A single density chain that keeps all its prefixes attains the reference on $39\%$ of the instances at $L=124$, and the enumerated seeds that keep only the endpoint of each chain on $49\%$. End-to-end execution success with the neural capability encoder is evaluated in \cref{sec:evaluation:e2}.

\subsection{End-to-end selection}
\label{app:e2e}

The neural encoder embeds the task prompt and each document with frozen \texttt{text-embedding-v4} and applies two linear maps into $d=64$ latent capabilities; supplies are a softplus of one map and demands a softmax of the other. It uses the fitted objective in \cref{eq:def-hat-F} with the response function in \cref{eq:app-response}.

Warm-up uses no executor. From the capability annotations of the other $13$ tasks, the $11$ training tasks of \cref{app:testbed} and the two test tasks not evaluated here, we enumerate the minimal covering sets and train the maps for $400$ Adam steps at learning rate $5\times10^{-4}$ to score each cover above the sets that drop one of its documents or add one, and above the other sets of its task. Online training then runs $17$ passes over the other tasks' instances at budgets ranging from $500$ to $2{,}000$ tokens, with the added price of \cref{sec:evaluation:e2} set to $\rho=0$. Each step decodes a set with \BPS{} under Thompson sampling over an ensemble of ten encoders, observes one pass/fail rollout of that set, and updates the maps on the smoothed log loss at learning rate $10^{-3}$.

The evaluation is leave-one-task-out: for each of the eight test tasks we remove its instances from the stream, train, and decode its instances at the three budgets and three prices of \cref{sec:evaluation:e2}, so the reported selections come from encoders that never received feedback on the task they select for. We repeat this with three training seeds and average each learned selector over them. On these neural objectives \BPS{} attains the exact optimum: enumerating every feasible subset returns its set on all $216$ combinations of instance, seed and price.

We execute every set chosen by any rule on any instance, $125$ distinct task--set pairs, with $16$ rollouts each. A rule's success is its sets' pass rate averaged over the $24$ instances, and its tokens are the summed document lengths of its sets averaged the same way. \Cref{tab:e2e} lists every point of \cref{fig:selection}.

\begin{table}[h]
\centering
\caption{Every point of \cref{fig:selection}: measured success and injected tokens per budget, and their mean over the three budgets. Each cell reads success / tokens; the self-selecting executor's tokens include the catalogue lines it reads.}
\label{tab:e2e}
\small
\begin{tabular}{lcccc}
\toprule
Rule & $B=500$ & $B=750$ & $B=1{,}000$ & Mean \\
\midrule
\BPS{}, $\rho=0$ & 0.66 / 464 & 0.65 / 686 & 0.70 / 927 & 0.67 / 693 \\
\BPS{}, $\rho=1$ & 0.62 / 456 & 0.66 / 513 & 0.66 / 550 & 0.65 / 506 \\
\BPS{}, $\rho=2$ & 0.49 / 370 & 0.52 / 410 & 0.62 / 456 & 0.55 / 412 \\
Executor self-select & 0.51 / 530 & 0.57 / 660 & 0.71 / 753 & 0.60 / 647 \\
Fine-tuned retriever, $k=1$ & 0.00 / 222 & 0.00 / 222 & 0.04 / 302 & 0.01 / 249 \\
Fine-tuned retriever, $k=2$ & 0.66 / 464 & 0.48 / 523 & 0.44 / 565 & 0.53 / 518 \\
Fine-tuned retriever, $k=3$ & 0.66 / 464 & 0.48 / 686 & 0.64 / 909 & 0.60 / 687 \\
SkillRouter, $k=1$ & 0.00 / 373 & 0.00 / 520 & 0.00 / 520 & 0.00 / 471 \\
SkillRouter, $k=2$ & 0.26 / 453 & 0.12 / 604 & 0.27 / 822 & 0.22 / 626 \\
SkillRouter, $k=3$ & 0.26 / 453 & 0.12 / 604 & 0.27 / 878 & 0.22 / 645 \\
Graph-of-Skills & 0.12 / 354 & 0.06 / 569 & 0.00 / 846 & 0.06 / 590 \\
BM25 & 0.00 / 471 & 0.12 / 691 & 0.12 / 934 & 0.08 / 699 \\
\bottomrule
\end{tabular}
\end{table}

Read on the same rounds, the lookup tables of \cref{sec:evaluation:e1}, fit on the test tasks, reach $0.70$ success on $695$ tokens. They mark the fitted objective's ceiling on these instances rather than a selector that a deployment could run.

The deployed systems rank the library once per task. BM25 runs over the full document bodies with the task prompt as query. SkillRouter runs the authors' code with their released $0.6$B embedding and reranker models, the reranker reordering the embedding's top $20$. Graph-of-Skills runs the authors' code on a graph built from our $31$ documents and fills our token budget in its ranking's order, skipping documents that do not fit. The self-selecting executor is shown the task, the $31$ catalogue lines and the budget and returns a list, truncated to the budget in its order. Budgeted top-$k$ walks a ranking from the top, skips any document that would exceed the budget, and stops after $k$ documents; both curves of \cref{fig:selection} saturate by $k=3$.

The fine-tuned retriever scores each document alone with a bilinear head on the same frozen \texttt{text-embedding-v4} embeddings, two linear maps into $64$ dimensions with a scale and an offset. Warm-up trains it for $400$ Adam steps at learning rate $5\times10^{-4}$ on the covering sets of the same $13$ training tasks, labelling the documents of a cover positive and the rest of the library negative. It then runs its own online phase on the same stream and the same $17$ passes as the encoder, one rollout per step at learning rate $10^{-3}$: it fills the budget in rank order, executes the set, and updates every document in it toward the set's pass/fail outcome under the smoothed log loss. The evaluation is the same leave-one-task-out protocol, and its selections are read at budgeted top-$k$.

\section{Limitations and Future Directions}
\label{app:future-work}

\subsection{Nonlinear context penalties}
\label{app:future-penalties}

The penalty $c_E(S)=\kappa_E\ell(S)$ in \cref{sec:model:objective} assigns the same marginal cost to every injected token. This is a first-order approximation of score degradation, expressed in the same units as the capability benefit. A nonlinear penalty $g_E(\ell(S))$, with $g_E(0)=0$ and $g_E$ nondecreasing, would allow the marginal cost to vary with the total injected length. For example, a convex penalty could represent degradation that becomes steeper as context grows, while a saturating penalty could represent a diminishing effect of additional length. Which form is appropriate for a given executor remains an empirical question.

Evaluating these alternatives would require separating length effects from changes in useful capability information. One approach would hold the required API information fixed while varying the amount of redundant or unrelated text, then compare linear and nonlinear penalties on held-out execution records. The evaluation should measure both prediction quality and the execution success of the selected sets under the same token budgets. Better prediction alone would not establish that a more flexible penalty improves selection. On the theoretical side, \cref{thm:main} relies on a penalty proportional to token length. Extending the guarantee requires identifying nonlinear penalty families for which efficient selection algorithms can still provide a useful bound on benefit and cost.

\subsection{Interactions beyond capability coverage}
\label{app:future-interactions}

In \cref{eq:def-G}, each skill has a fixed capability supply vector. We add supplies within each dimension and apply $h_k$, then weight coverage by task demand. The nonlinear response makes each skill's marginal benefit depend on the selected set, capturing redundancy and complementarity. Additional dimensions can represent knowledge needed to combine capabilities, such as how to connect two APIs.

A limitation is that this representation does not explicitly model conflicts or confusion between particular skills. Two useful documents may prescribe incompatible conventions or describe easily confused operations, so their joint injection could reduce execution success even when capability coverage is sufficient.

We make this simplification because we assume a sufficiently capable executor that can distinguish different skills' roles and resolve common ambiguities from their documentation. Under this assumption, residual adverse interactions are limited enough for the coverage model to be useful. Its adequacy depends on the executor and the task.

Future work could model these residual effects through interaction terms learned from execution feedback, then study how to optimize the extended objective under a token budget. Because such terms may violate diminishing returns, obtaining provable guarantees requires identifying suitable restrictions on the interactions or developing new selection algorithms.

\end{CJK*}
\end{document}